\documentclass[twoside]{article}

\usepackage[accepted]{aistats2024}

\usepackage[round]{natbib}

\usepackage[usenames,dvipsnames]{xcolor}

\usepackage{amsmath}
\usepackage{amssymb}
\usepackage{amsthm}
\usepackage{mathtools}
\usepackage{dsfont}
\usepackage{mathrsfs}

\usepackage{paralist}
\usepackage{algorithm}
\usepackage{algorithmicx}
\usepackage[noend]{algpseudocode}
\algrenewcommand\textproc{}%

\usepackage{thm-restate}
\usepackage{xspace}

\usepackage{float}
\usepackage{wrapfig}
\usepackage{caption}
\usepackage{subcaption}

\usepackage{minitoc}

\newcommand{\AlgNameLong}{Convex Constraint Learning for Reinforcement Learning\xspace}
\newcommand{\AlgNameShort}{\texttt{CoCoRL}\xspace}

\newcommand{\githublink}{\url{https://github.com/lasgroup/cocorl}}

\renewcommand{\paragraph}[1]{\textbf{{#1}}}

\usepackage{xkcdcolors}

\newcommand{\CoCoRLcolor}{xkcdBrightRed}

\newcommand{\MaxMarginIRLVanillacolor}{xkcdBrightPurple}
\newcommand{\MaxMarginIRLSharedRewardcolor}{xkcdBurntRed}
\newcommand{\MaxMarginIRLKnownRewardcolor}{xkcdDarkLimeGreen}

\newcommand{\legendCocorlAbstained}{\raisebox{-0.3ex}{\includegraphics[width=2ex]{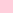}}}%

\usepackage{hyperref}
\usepackage[nameinlink]{cleveref}  %

\setcitestyle{
    braces={round}
}
\definecolor{mydarkblue}{rgb}{0,0.08,0.45}
\hypersetup{
  colorlinks   = true, %
  allcolors    = mydarkblue,
}

\newcommand{\reals}{\mathbb{R}}
\newcommand{\Gaussian}{\mathcal{N}}

\newcommand{\argmax}{\mathop\mathrm{argmax}}
\newcommand{\argmin}{\mathop\mathrm{argmin}}
\newcommand{\Expectation}{\mathbb{E}}

\newcommand{\defeq}{\coloneqq}

\newcommand{\truesafeset}{\mathcal{F}}
\newcommand{\safeset}{\mathcal{S}}
\newcommand{\unsafeset}{\mathcal{U}}
\newcommand{\conv}{\mathtt{conv}}

\newcommand{\indicator}[1]{\mathds{1}_{\{#1\}}}

\newcommand{\dif}{\mathrm{d}}

\newcommand{\state}{s}
\newcommand{\action}{a}

\newcommand{\StateSpace}{S}
\newcommand{\ActionSpace}{A}
\newcommand{\StateSpaceSize}{|S|}

\newcommand{\TransitionModel}{P}

\newcommand{\DiscountFactor}{\gamma}
\newcommand{\discount}{\DiscountFactor}

\newcommand{\reward}{r}
\newcommand{\MDP}{\mathcal{M}}

\newcommand{\initdist}{\mu_0}

\newcommand{\PolicyReturn}{G}

\newcommand{\CMDP}{\mathcal{C}}
\newcommand{\cost}{c}

\newcommand{\CumulativeCost}{J}
\newcommand{\threshold}{\xi}

\newcommand{\policy}{\pi}

\newcommand{\demonstrations}{\mathcal{D}}

\newcommand{\evalreward}{\reward_{\text{eval}}}

\newcommand{\FeatureFunction}{\mathbf{f}}

\newcommand{\Occupancy}{\mu}

\newcommand{\regret}{\mathcal{R}}

\newcommand{\diag}{\mathrm{diag}}

\newcommand{\vv}{\mathbf{v}}
\newcommand{\vV}{\mathbf{V}}
\newcommand{\vQ}{\mathbf{Q}}
\newcommand{\vP}{\mathbf{P}}
\newcommand{\vr}{\mathbf{r}}
\newcommand{\vc}{\mathbf{c}}
\newcommand{\vb}{\mathbf{b}}
\newcommand{\vw}{\mathbf{w}}
\newcommand{\vu}{\mathbf{u}}

\newcommand{\trajectory}{\tau}

\newcommand{\case}[1]{\textbf{Case #1}}

\newcommand{\rank}{\mathrm{rank}}

\newcommand{\NumDemos}{k}
\newcommand{\NumConst}{n}
\newcommand{\NumInferredConst}{m}
\newcommand{\NumTraj}{n_{\text{traj}}}

\newcommand{\bigO}{\mathcal{O}}

\theoremstyle{plain}

\newtheorem{theorem}{Theorem}

\newtheorem{lemma}{Lemma}

\theoremstyle{definition}
\newtheorem{assumption}{Assumption}

\crefname{assumption}{assumption}{assumptions}
\Crefname{assumption}{Assumption}{Assumptions}

\declaretheoremstyle[%
  spaceabove=-6pt,%
  spacebelow=0pt,%
  headfont=\normalfont\itshape,%
  postheadspace=1em,%
  qed=\qedsymbol%
]{prfstyle}

\begin{document}

\runningtitle{Learning Safety Constraints from Demonstrations with Unknown Constraints}

\twocolumn[

\aistatstitle{Learning Safety Constraints from \\ Demonstrations with Unknown Rewards}

\aistatsauthor{ David Lindner \And Xin Chen \And  Sebastian Tschiatschek \AND Katja Hofmann \And Andreas Krause }

\vspace{-3em}
\aistatsaddress{ ETH Zurich \And ETH Zurich \And University of Vienna \AND Microsoft Research, Cambridge \And ETH Zurich } ]

\begin{abstract}
We propose {\em \AlgNameLong} (\AlgNameShort), a novel approach for inferring shared constraints in a Constrained Markov Decision Process (CMDP) from a set of safe demonstrations with possibly different reward functions. 
While previous work is limited to demonstrations with known rewards or fully known environment dynamics, \AlgNameShort can learn constraints from demonstrations with \emph{different unknown rewards} without knowledge of the environment dynamics.
\AlgNameShort constructs a convex safe set based on demonstrations, which provably guarantees safety even for potentially sub-optimal (but safe) demonstrations. For near-optimal demonstrations, \AlgNameShort converges to the true safe set with no policy regret.
We evaluate \AlgNameShort in gridworld environments and a driving simulation with multiple constraints. \AlgNameShort learns constraints that lead to safe driving behavior. Importantly, we can safely transfer the learned constraints to different tasks and environments. In contrast, alternative methods based on Inverse Reinforcement Learning (IRL) often exhibit poor performance and learn unsafe policies.
\end{abstract}

\section{Introduction}\label{sec:introduction}

\begin{figure*}
    \centering
    \includegraphics[width=0.9\linewidth]{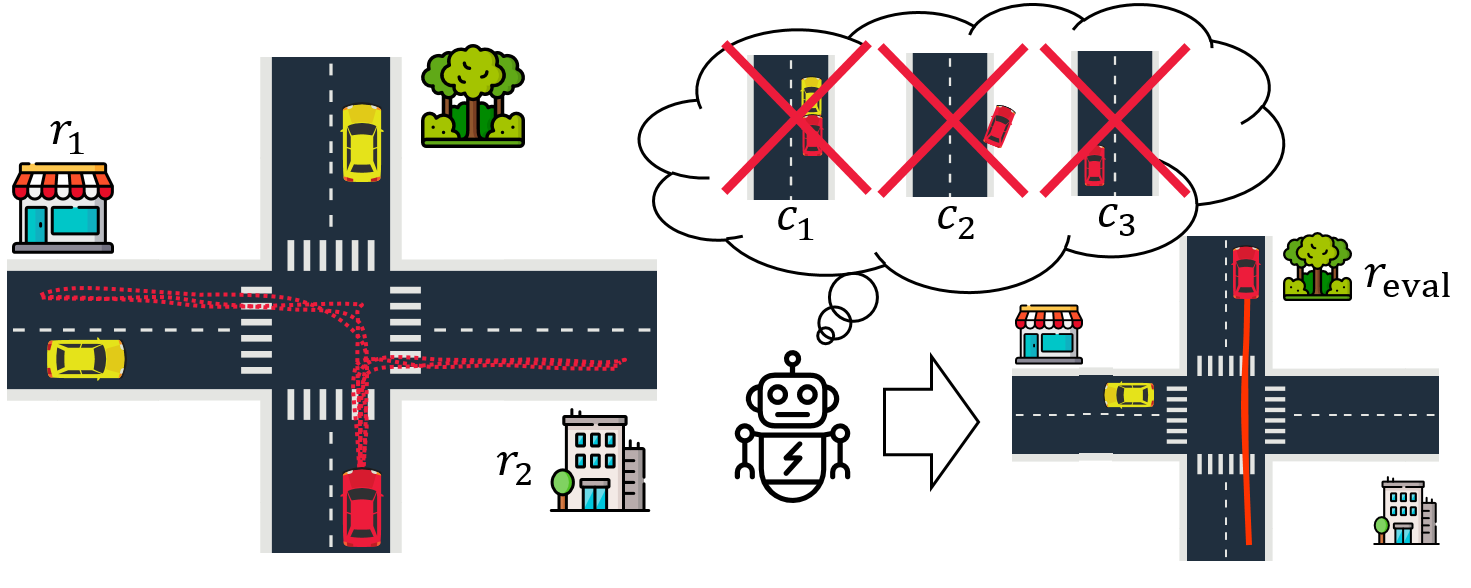}
    \\
    \hspace{0.9cm} \textbf{Training} \hspace{7.5cm} \textbf{Evaluation}
    \caption{
    \AlgNameShort can, e.g., learn safe driving behavior from diverse driving trajectories with different unknown reward functions (here, $r_1$: ``turn left'', $r_2$: ``turn right'') . It infers constraints $\cost_1, \cost_2, \cost_3$ describing desirable driving behavior from demonstrations without knowledge of the specific reward functions $\reward_1, \reward_2$. These inferred constraints allow to optimize for a new reward function $\reward_{\text{eval}}$ (``go straight''), ensuring safe driving behavior even in \emph{new situations} where matching demonstrations are unavailable.}
    \label{fig:headline_figure}
\end{figure*}

Constrained Markov Decision Processes (CMDPs) integrate safety constraints with reward maximization in reinforcement learning (RL). However, similar to reward functions, it can be difficult to specify constraint functions manually~\citep{krakovna2020specification}.
To tackle this issue, recent work proposes learning models of the constraints~\citep[e.g.,][]{scobee2019maximum,anwar2020inverse} analogously to reward models~\citep{leike2018scalable}. However, existing approaches to constraint inference rely on demonstrations with \emph{known reward functions}, which is often unrealistic and limiting in real-world applications.

For instance, consider the domain of autonomous driving (\Cref{fig:headline_figure}), where specifying rewards/constraints is particularly difficult~\citep{knox2023reward}. However, we can gather diverse human driving trajectories that satisfy shared constraints, such as keeping an appropriate distance from other cars and avoiding crashes. These trajectories will usually have different (unknown) routes or driving style preferences, which we can model as different reward functions. In such scenarios, we aim to infer constraints from demonstrations with {\em shared constraints} but \emph{unknown rewards}.

\paragraph{Contributions.} \looseness -1 In this paper: (1) we
introduce the problem of inferring constraints in CMDPs from demonstrations with unknown rewards (\Cref{sec:problem_setup}); (2) we introduce {\em \AlgNameLong} (\AlgNameShort), a novel method for addressing this problem (\Cref{sec:method}); (3) we prove that \AlgNameShort guarantees safety (\Cref{sec:safe_set}), and, for (approximately) optimal demonstrations, asymptotic optimality (\Cref{sec:convergence}), while IRL provably cannot guarantee safety (\Cref{sec:irl_limitations}); and, (4) we conduct comprehensive empirical evaluations of \AlgNameShort in tabular environments and a continuous driving task with multiple constraints (\Cref{sec:experiments}). \AlgNameShort learns constraints that lead to safe driving behavior and that can be transferred across different tasks and environments. In contrast, methods based on Inverse Reinforcement Learning (IRL) often perform poorly and lack safety guarantees, and methods based on Imitation Learning (IL) lack robustness and transferability. 

Overall, our theoretical and empirical results illustrate the potential of \AlgNameShort for constraint inference and the advantage of employing CMDPs with inferred constraints instead of relying on MDPs with IRL in safety-critical applications.
We provide an open-source implementation of \AlgNameShort and the code necessary to reproduce all experiments at \githublink.

\section{Background}\label{sec:background}

\paragraph{Markov Decision Processes (MDPs).}
An (infinite-horizon, discounted) MDP \citep{puterman1990markov} is a tuple $(\StateSpace, \ActionSpace, \TransitionModel, \initdist, \DiscountFactor, \reward)$, where $\StateSpace$ is the state space, $\ActionSpace$ the action space, $\TransitionModel: \StateSpace \times \ActionSpace \times \StateSpace \to [0, 1]$ the transition model, $\initdist: \StateSpace \to [0, 1]$ the initial state distribution, $\DiscountFactor \in (0, 1)$ the discount factor, and $\reward: \StateSpace \times \ActionSpace \to [0, 1]$ the reward function.
An agent starts in a state $s \sim \initdist$, and takes actions sampled from a policy $\action \sim \policy(\action | \state)$ that lead to a next state according to the transition kernel $\TransitionModel(s' | s, a)$. The agent aims to maximize the expected discounted return
$
\PolicyReturn_\reward(\policy) = \Expectation_{\TransitionModel,\policy} \left[ \sum_{t=0}^\infty \DiscountFactor^t \reward(s_t, a_t) \right]
$.
We often use the (discounted) occupancy measure
$
\Occupancy_\policy(\state, \action) = \Expectation_{\TransitionModel,\policy} \left[ \sum_{t=0}^\infty \DiscountFactor^t \indicator{\state_t=\state, \action_t=\action} \right]
$.
For finite state and action spaces, the return is linear in $\Occupancy_\policy$, i.e., 
$
\PolicyReturn_{\reward}(\policy) = \sum_{\state,\action} \Occupancy_\policy(\state, \action) \reward(\state, \action)
$.

\paragraph{Constrained MDPs (CMDPs).}
A CMDP \citep{altman1999constrained} extends an MDP with a set of cost functions $\cost_1, \dots, \cost_\NumConst$ defined similar to the reward function as $\cost_j: \StateSpace \times \ActionSpace \to [0, 1]$ and thresholds $\threshold_1, \dots, \threshold_\NumConst$. The agent's goal is to maximize $\PolicyReturn_{\reward}(\policy)$ under constraints on the expected discounted cumulative costs
$
\CumulativeCost_j(\pi) = \Expectation_{\TransitionModel,\policy} \left[ \sum_{t=0}^\infty \DiscountFactor^t \cost_j(s_t, a_t) \right]
$. In particular, a feasible policy must satisfy $\CumulativeCost_j(\pi) \leq \threshold_j$ for all $j$. Again, we can write the cumulative costs as
$
\CumulativeCost_j(\pi) = \sum_{\state,\action} \Occupancy_\policy(\state, \action) \cost_j(\state, \action)
$.

\paragraph{Inverse Reinforcement Learning (IRL).}
IRL aims to recover an MDP's reward function from expert demonstrations, ensuring the demonstrations are optimal under the identified reward function. It seeks a reward function where the optimal policy aligns with the expert demonstrations' feature expectations \citep{abbeel2004apprenticeship}. Since the IRL problem is underspecified, algorithms differ in selecting a unique solution, with popular approaches including maximum margin IRL \citep{ng2000algorithms} and maximum (causal) entropy IRL \citep{ziebart2008maximum}.

\section{Problem Setup and IRL-based Approaches}

In this section, we first introduce the problem of inferring constraints from demonstrations in a CMDP. Then, we discuss why naive solutions based on IRL fail to solve this problem.

\subsection{Problem Setup}\label{sec:problem_setup}

We consider an infinite-horizon discounted CMDP without reward function and constraints, denoted by $(\StateSpace, \ActionSpace, \TransitionModel, \initdist, \DiscountFactor)$.
We have access to demonstrations from $\NumDemos$ safe, i.e., feasible, policies $\demonstrations = \{ \policy_1^*, \dots, \policy_\NumDemos^* \}$. The demonstrations have $\NumDemos$ unknown reward functions $\reward_1, \dots, \reward_\NumDemos$, but share the same constraints defined by unknown cost functions $\cost_1, \dots, \cost_\NumConst$ and thresholds $\threshold_1, \dots, \threshold_\NumConst$. 
We assume policy $\policy_i^*$ is feasible in the CMDP $(\StateSpace, \ActionSpace, \TransitionModel, \initdist, \DiscountFactor, \reward_i, \{\cost_j\}_{j=1}^\NumConst, \{\threshold_j\}_{j=1}^\NumConst )$. Generally, we assume policy $\policy_i^*$ optimizes for reward $\reward_i$, but it does not need to be optimal.

In our driving example (\Cref{fig:headline_figure}), the rewards of the demonstrations can correspond to different routes and driver preferences, whereas the shared constraints describe driving rules and safety-critical behaviors, such as maintaining an appropriate distance to other cars or staying in the correct lane.

We assume that cost functions are linear in a  $d$-dimensional feature space,%
\footnote{%
Note that in the case of a finite state and action space, this assumption is \emph{w.l.o.g.}, because we can use the state-action occupancy measure $\Occupancy_\policy(\state, \action)$ as features.
}
represented by $\FeatureFunction: \StateSpace \times \ActionSpace \to [0, 1]^d$.
We represent demonstrations by their discounted feature expectations (slightly overloading notation) defined as $\FeatureFunction(\policy) = \Expectation_{\TransitionModel,\policy} [\sum_{t=0}^\infty \DiscountFactor^t \FeatureFunction(s_t, a_t)]$.
We then assume the cost functions are $\cost_j(\state, \action) = \phi_j^T \FeatureFunction(\state, \action)$, where each cost function has a parameter vector $\phi_j \in \reals^d$. For now, we assume we know the exact feature expectations of the demonstrations. In \Cref{sec:estimating_feature_expectations}, we discuss sample-based estimation.
Our approach does not require linear reward functions in general.

For evaluation, we receive a new reward function $\evalreward$ and aim to find an optimal policy for the CMDP $(\StateSpace, \ActionSpace, \TransitionModel, \initdist, \DiscountFactor, \evalreward, \{\cost_j\}_{j=1}^\NumConst, \{\threshold_j\}_{j=1}^\NumConst )$ where the constraints are unknown. Our goal is to ensure safety while achieving high reward under $\reward_{\text{eval}}$.

\subsection{Limitations of IRL in CMDPs}
\label{sec:irl_limitations}

One might be tempted to disregard the CMDP nature of the problem and try to apply standard IRL methods to infer a reward from demonstrations, and much prior work on constraint learning primarily extends work from IRL (e.g., \citet{scobee2019maximum,stocking2021discretizing,anwar2020inverse}).
However, IRL poses at least two key problems in our setting: (1) IRL alone cannot guarantee safety, and (2) it is unclear how to learn a transferable representation of the constraints.

To elaborate, let us first assume we have a single expert demonstration from a CMDP and apply IRL to infer a reward function, assuming no constraints. Then, in some CMDPs, any reward functions IRL can infer can yield unsafe optimal policies.
\begin{restatable}[IRL can be unsafe]{proposition}{IRLUnsafe}\label{thm:irl-unsafe}
There are CMDPs $\CMDP = (\StateSpace, \ActionSpace, \TransitionModel, \initdist, \discount, \reward, \{\cost_j\}_{j=1}^\NumConst, \{\threshold_j\}_{j=1}^\NumConst )$ such that for any optimal policy $\policy^*$ in $\CMDP$ and any reward function $\reward_{\text{IRL}}$ that could be returned by an IRL algorithm, the resulting MDP $(\StateSpace, \ActionSpace, \TransitionModel, \initdist, \discount, \reward_{\text{IRL}})$ has optimal policies that are unsafe in $\CMDP$.
\end{restatable}
This follows from CMDPs having only stochastic optimal policies, and MDPs always having a deterministic optimal policy \citep{puterman1990markov,altman1999constrained}.%
\footnote{See \Cref{app:proofs} for a full proof of this and other theoretical results in this paper.}

Next let us consider our actually problem setting, where we have a set of demonstrations that share constraints. If we use IRL, it is unclear how to learn the shared part and transfer it to a new task. A first approach could be to learn as a \emph{shared} reward penalty. Unfortunately, it turns out learning a shared penalty can be infeasible, even with known rewards.
\begin{restatable}{proposition}{IRLKnownRewardFails}\label{thm:irl-known-reward-fails}
Let $\CMDP = (\StateSpace,\allowbreak \ActionSpace,\allowbreak \TransitionModel,\allowbreak \initdist,\allowbreak \discount,\allowbreak \{\cost_j\}_{j=1}^\NumConst,\allowbreak \{\threshold_j\}_{j=1}^\NumConst)$ be a CMDP without reward. Let $\reward_1, \reward_2$ be two reward functions and $\policy_1^*$ and $\policy_2^*$ corresponding optimal policies in $\CMDP \cup \{\reward_1\}$ and $\CMDP \cup \{\reward_2\}$.
Let $\MDP = (\StateSpace, \ActionSpace, \TransitionModel, \initdist, \DiscountFactor)$ be the corresponding MDP without reward. Without additional assumptions, we cannot guarantee the existence of a function $\hat{c}: \StateSpace \times \ActionSpace \to \mathbb{R}$ such that $\policy_1^*$ is optimal in the MDP $\mathcal{M} \cup \{ \reward_1 + \hat{c} \}$ and $\policy_2^*$ is optimal in the MDP $\mathcal{M} \cup \{ \reward_2 + \hat{c} \}$.
\end{restatable}
Similar to \Cref{thm:irl-unsafe}, this result stems from the fundamental difference between MDPs and CMDPs.
These findings suggest that IRL is not a suitable approach for learning from demonstrations in a CMDP. Our empirical findings in \Cref{sec:experiments} confirm this.

\section{Convex Constraint Learning for Reinforcement Learning (\AlgNameShort)}
\label{sec:method}

We are now ready to introduce the \AlgNameShort algorithm. In essence, \AlgNameShort consists of three steps:
\begin{compactenum}
\item Construct a conservative safe set $\safeset$ from the demonstrations $\demonstrations$.
\item Construct an \emph{inferred CMDP} from the original CMDP and the safe set $\safeset$.
\item Use a standard constrained RL algorithm to solve the inferred CMDP and return resulting policy.
\end{compactenum}

In this section, we (1) introduce the key idea of using convexity to construct provably conservative safe set and the inferred CMDP (\Cref{sec:safe_set}); (2) establish convergence results under (approximately) optimal demonstrations (\Cref{sec:convergence}); (3) introduce several approaches for using estimated feature expectations (\Cref{sec:estimating_feature_expectations}); and, (4) discuss practical considerations when implementing \AlgNameShort (\Cref{sec:implementation}).

Importantly, to construct the conservative safe set, we only require safe demonstrations. For convergence, we make certain optimality assumptions on the demonstrations, either exact optimality or Boltzmann-rationality. \Cref{app:proofs} contains all omitted proofs.

\subsection{Constructing a Convex Safe Set}\label{sec:safe_set}

Let the \emph{true safe set} $\truesafeset = \{ \policy | \forall j: \CumulativeCost_j(\policy) \leq \threshold_j \}$ be the set of all policies that satisfy the constraints. \AlgNameShort is based on the key observation that $\truesafeset$ is convex.
\begin{restatable}[$\truesafeset$ is convex]{lemma}{TrueSafeSetConvex}\label{lem:true-safe-set-convex}
For any CMDP, suppose $\policy_1, \policy_2 \in \truesafeset = \{ \policy | \forall j: \CumulativeCost_j(\policy) \leq \threshold_j \}$. Let $\bar{\policy}_{12}$ be a mixture policy such that $\FeatureFunction(\bar{\policy}_{12}) = \lambda \FeatureFunction(\policy_1) + (1-\lambda) \FeatureFunction(\policy_2)$ with $\lambda \in [0, 1]$. Then, we have $\bar{\policy}_{12} \in \truesafeset$.
\end{restatable}

We know that all demonstrations are safe in the original CMDP and convex combinations of their feature expectations are safe. This insight leads us to a natural approach for constructing a conservative safe set: create the convex hull of the feature expectations of the demonstrations:
\[
\safeset \defeq \left\{ \policy \middle| \FeatureFunction(\policy) = \sum_{i=1}^\NumDemos \lambda_i \FeatureFunction(\policy^*_i), ~ \lambda_i \geq 0 \text{ and } \sum_{i=1}^k \lambda_i = 1 \right\}
\]
Thanks to the convexity of $\truesafeset$, we can now guarantee safety for all policies $\policy\in\safeset$.
\begin{restatable}[Estimated safe set]{theorem}{SafeSetIsSafe}\label{thm:safe-set-is-safe}
Any policy $\policy \in \safeset$ is safe, i.e., $\policy \in \truesafeset$.
\end{restatable}
During evaluation, we want to find an optimal policy in our coservative safe set for a new reward function $\evalreward$, i.e., we want to solve
$
\max_{\policy\in \safeset} \PolicyReturn_{\evalreward}(\policy)
$.
Conveniently, we can reduce this problem to solving a new CMDP with identical dynamics. Specifically, we can find a set of linear constraints to represent the safe set $\safeset$. This result follows from standard convex analysis: $\safeset$ is a convex polyhedron, which is the solution to a set of linear equations. Constructing a convex hull from a set of points and identifying the linear equations is a standard problem in polyhedral combinatorics.

Consequently, we can optimize within our safe set $\safeset$ by solving an ``inferred'' CMDP.
\begin{restatable}[Inferred CMDP]{theorem}{EstimatedCMDP}\label{thm:estimated-cmdp}
We can find cost functions $\{ \hat{c}_j \}_{j=1}^\NumInferredConst$ and thresholds $\{ \hat{\threshold}_j \}_{j=1}^\NumInferredConst$ such that for any reward function $\evalreward$, solving the inferred CMDP $(\StateSpace, \ActionSpace, \TransitionModel, \initdist, \DiscountFactor, \evalreward, \{ \hat{c}_j \}_{j=1}^\NumInferredConst, \{ \hat{\threshold}_j \}_{j=1}^\NumInferredConst )$ is equivalent to finding
$
\policy^* \in \argmax_{\policy \in \safeset} \PolicyReturn_{{\reward}_\text{eval}}(\policy)
$.
Consequently, if $\safeset$ contains an optimal policy for the true CMDP, solving the inferred CMDP returns a policy that is also optimal in the true CMDP.
\end{restatable}
Note that the number of inferred constraints $\NumInferredConst$ is not necessarily equal to the number of true constraints $\NumConst$, but we do not assume to know $\NumConst$.
It immediately follows from \Cref{thm:estimated-cmdp} that our safe set is worst-case optimal because, without additional assumptions, we cannot dismiss the possibility that the true CMDP coincides with our inferred CMDP.
\begin{restatable}[$\safeset$ is maximal]{corollary}{SafeSetOptimal}\label{thm:safe-set-optimal}
If $\policy \notin \safeset$, there exist $\reward_1, \dots, \reward_\NumDemos$ and $\cost_1, \dots, \cost_\NumConst$, such that
the expert policies $\policy^*_1, \dots, \policy^*_\NumDemos$ are optimal in the CMDPs $(\StateSpace, \ActionSpace, \TransitionModel, \initdist, \DiscountFactor, \reward_i, \{ \cost_j \}_{j=1}^\NumConst, \{ \threshold_j \}_{j=1}^\NumConst)$, but $\policy \notin \truesafeset$.
\end{restatable}
In essence, $\safeset$ is the largest set we can choose while ensuring safety.

\subsection{Convergence for (Approximately) Optimal Demonstrations}\label{sec:convergence}

So far, we studied the safety of \AlgNameShort. In this section, we aim to establish its convergence.  To this end, we compare the solution returned by \AlgNameShort, i.e., $\max_{\policy \in \safeset} \PolicyReturn_{\reward}(\policy)$, to the optimal safe solution, i.e., $\max_{\policy \in \truesafeset} \PolicyReturn_{\reward}(\policy)$, for a given evaluation reward $\evalreward$. In particular, we define the \emph{policy regret}
\[
\regret(\reward, \safeset) = \max_{\policy \in \truesafeset} \PolicyReturn_{\reward}(\policy) - \max_{\policy \in \safeset} \PolicyReturn_{\reward}(\policy).
\]
Let us assume the reward functions of the demonstrations and the reward functions during evaluation follow the same distribution $P(\reward)$. After observing $\NumDemos$ demonstrations, we construct the safe set $\safeset_\NumDemos$, and consider the expectation $\Expectation_\reward[\regret(\reward, \safeset_\NumDemos)]$ under the reward distribution.
Now, we aim to show that regret is zero asymptotically, i.e.,
$
\lim_{\NumDemos\to\infty} \Expectation_\reward \left[ \regret(\reward, \safeset_\NumDemos) \right] = 0
$.

In order to establish convergence for \AlgNameShort, we need additional assumptions. In particular, we need the demonstrations to be diverse enough to ensure we learn a sufficiently large safe set.
One way to get this diversity is by assuming the demonstrations are (approximately) optimal for a set of \emph{unknown} reward function.

In particular, we establish \AlgNameShort's convergence to optimality under two conditions: (1) when the safe demonstrations are exactly optimal, and (2) when they are  approximately optimal according to a Boltzmann model. To simplify the convergence analysis, we additionally assume the reward function are linear.%
\footnote{%
The linear reward assumption could likely be relaxed with more careful analysis, and the analysis could be extended to other types of optimality assumptions.
}
Importantly, we need these assumptions only for the convergence analysis in this section. They are not necessary for \AlgNameShort to guarantee safety (see \Cref{sec:safe_set}) or work in practice (see \Cref{sec:experiments}).

First, let us consider exactly optimal demonstrations.

\begin{assumption}[Linear rewards]\label{ass:linear_rewards}
The demonstrations reward functions and evaluation reward functions are linear. In particular $\reward_i(\state, \action) = \theta_i^T \FeatureFunction(\state, \action)$ and $\evalreward(\state, \action) = \theta_{\text{eval}}^T \FeatureFunction(\state, \action)$, where $\theta_i \in \reals^d$ and $\theta_{\text{eval}}^T \in \reals^d$.
\end{assumption}

\begin{assumption}[Exact optimality]\label{ass:exact-optimality}
\looseness -1
Each $\policy_i^*$ is optimal in the CMDP $(\StateSpace, \ActionSpace, \TransitionModel, \initdist, \DiscountFactor, \reward_i, \{ \cost_j \}_{j}, \{ \threshold_j \}_{j})$.
\end{assumption}
\begin{restatable}[Convergence, exact optimality]{theorem}{ConvergenceNoiseFree}\label{thm:convergence-noise-free}
Under \Cref{ass:linear_rewards,ass:exact-optimality}, for any $\delta > 0$, after
$
\NumDemos \geq \log(\delta / f_v(d, \NumConst)) / \log(1 - \delta/f_v(d, \NumConst))
$, we have $P(\regret(\reward, \safeset_\NumDemos) > 0) \leq \delta$, where $f_v(d, \NumConst)$ is an upper bound on the number of vertices of the true safe set. In particular,
$
\lim_{\NumDemos\to\infty} \Expectation_\reward \left[ \regret(\reward, \safeset_\NumDemos) \right] = 0
$.
\end{restatable}
Due to both the true safe set and the estimated safe set being convex polyhedra, we have a finite hypothesis class consisting of all convex polyhedra with vertices that are a subset of the true safe set's vertices. The proof builds on the insight that all vertices supported by the distribution induced by $P(\reward)$ will eventually be observed, while unsupported vertices do not contribute to the regret.
The number of vertices is typically on the order $f_v(d, \NumConst) \in \bigO(n^{\lfloor d/2 \rfloor})$. A tighter bound is given, for example, by \citet{mcmullen1970maximum}.

In practical settings, it is unrealistic to assume perfectly optimal demonstrations. Thus, we relax this assumption and consider the case where demonstrations are only approximately optimal.
\begin{assumption}[Boltzmann-rationality]\label{ass:boltzmann-noise}
We observe demonstrations following a Boltzmann policy distribution
$
\policy_i^* \sim \exp\left( \beta \PolicyReturn_{\reward_i}(\policy) \right) \indicator{\CumulativeCost_1(\policy) \leq \threshold_1, \dots, \CumulativeCost_\NumConst(\policy) \leq \threshold_\NumConst} / Z_i
$, where $\PolicyReturn_{\reward_i}(\policy)$ is the expected return computed w.r.t.\ the reward $\reward_i$, and $Z_i$ is a normalization constant.
\end{assumption}
\begin{restatable}[Convergence, Boltzmann-rationality]{theorem}{ConvergenceBoltzmann}\label{thm:convergence-boltzmann}
Under \Cref{ass:boltzmann-noise}, for any $\delta > 0$, after
$
\NumDemos \geq \log(\delta / f_v(d, \NumConst)) / \log(1 - \exp(- \beta d / (1-\gamma)))
$,
we have $P(\regret(\reward, \safeset_\NumDemos) > 0) \leq \delta$, where $f_v(d, \NumConst)$ is an upper bound on the number of vertices of the true safe set. In particular,
$
\lim_{\NumDemos\to\infty} \Expectation_\reward \left[ \regret(\reward, \safeset_\NumDemos) \right] = 0
$.
\end{restatable}
This result relies on the Boltzmann distribution having a fixed lower bound over the bounded set of policy features. By enumerating the vertices of the true safe set, we can establish an upper bound on the probability of the ``bad event'' where a vertex is not adequately covered by the set of $\NumDemos$ demonstrations. This shows how the regret must decrease as $\NumDemos$ increases.

\subsection{Estimating Feature Expectations}\label{sec:estimating_feature_expectations}

So far, we assumed access to the true feature expectations of the demonstrations $\FeatureFunction(\policy_i^*)$. However, in practice, we often rely on samples from the environment to estimate the feature expectations.
Given $\NumTraj$ samples from policy $\policy_i^*$, we can estimate the feature expectations by taking the sample mean: $\hat{\FeatureFunction}(\policy_i^*) = \frac{1}{\NumTraj} \sum_{p=1}^{\NumTraj} \FeatureFunction(\trajectory_p^i)$, where $\FeatureFunction(\trajectory_p^i) = \sum_{t=0}^\infty \DiscountFactor^t \FeatureFunction(\state_t, \action_t)$ and $\trajectory_p^i = (\state_0, \action_0, \state_1, \action_1, \dots)$ represents a trajectory from rolling out $\policy_i^*$ in the environment. This estimate is unbiased, i.e., $\Expectation[\hat{\FeatureFunction}(\policy_i^*)] = \FeatureFunction(\policy_i^*)$, and we can use standard concentration inequalities to quantify its uncertainty.
This allows us to ensure $\epsilon$-safety when using the estimated feature expectations.
\begin{restatable}[$\epsilon$-safety with estimated feature expectations.]{theorem}{EstimatedFeaturesSafety}\label{thm:estimated-features-safety}
Suppose we estimate the feature expectations of each policy $\policy_i^*$ using at least $\NumTraj > d \log(\NumConst \NumDemos / \delta) / (2 \epsilon^2 (1 - \DiscountFactor) )$ samples, and construct the estimated safe set
$
\hat{\safeset} = \conv(\hat{\FeatureFunction}(\policy_1^*), \dots, \hat{\FeatureFunction}(\policy_\NumDemos^*))
$.
Then, we have
$
P( \max_j (\CumulativeCost_j(\policy) - \threshold_j) > \epsilon | \policy \in \hat{\safeset} ) < \delta
$.
\end{restatable}
We can extend our analysis to derive convergence bounds for exactly optimal or Boltzmann-rational demonstrations, similar to \Cref{thm:convergence-noise-free,thm:convergence-boltzmann}. The exact bounds are provided in \Cref{app:estimated_features_proofs}. Importantly, we get no regret asymptotically in both cases.

Alternatively, we could leverage confidence sets around $\hat{\FeatureFunction}(\policy_i^*)$ to construct a conservative safe set. This approach guarantees exact safety and convergence as long as the confidence intervals shrink. However, it involves constructing a \emph{guaranteed hull}~\citep{sember2011guarantees}, which can be computationally more expensive than constructing a simple convex hull. The guaranteed hull can also be overly conservative in practice. See \Cref{app:conservative_safe_set_from_estimated_features} for a detailed discussion. 

Crucially, \AlgNameShort can be applied in environments with continuous state and action spaces and nonlinear policy classes, as long as the feature expectations can be computed or estimated.

\subsection{Practical Implementation}\label{sec:implementation}

To implement \AlgNameShort, we first compute or estimate the demonstrations' feature expectations $\FeatureFunction(\policy_i^*)$. Then, we use the Quickhull algorithm~\citep{barber1996quickhull} to construct a convex hull. Finally, we solve the inferred CMDP from \Cref{thm:estimated-cmdp} using a constrained RL algorithm, the choice depending on the environment.
Constructing the convex hull only has negligible computational cost compared to solving the inferred CMDP. For very high-dimensional environments or large numbers of demonstrations it may become beneficial to approximate the convex hull instead.

We make two further extensions to enhance the robustness of this basic algorithm: we use projections to handle degenerate safe sets and we incrementally expand the safe set.
We discuss them briefly here, and in
\Cref{app:implementation_details} in more detail.

\paragraph{Handling degenerate safe sets.} The demonstrations might lie in a lower-dimensional subspace of $\reals^d$ (i.e., have a rank less than $d$). In that case, we project them onto this lower-dimensional subspace, construct the convex hull in that space, and then project back the resulting convex hull to $\reals^d$. This is beneficial because Quickhull is not specifically designed to handle degenerate convex hulls.

\paragraph{Incrementally expand the safe set.} If we construct $\safeset$ from all demonstrations, it will have many redundant vertices, which can result in numerical instabilities. To avoid this, we incrementally add points from $\demonstrations$ to $\safeset$. We greedily add the point furthest away from $\safeset$ until the distance of the furthest point is less than some $\varepsilon > 0$. In addition to mitigating numerical issues, this approach reduces the number of constraints in the inferred CMDP, making it easier to solve.
Importantly, we do not loose any safety guarantees.

\section{Related Work}\label{sec:related_work}

\paragraph{Constraint learning in RL.}
Previous work on safe RL typically assumes fully known safety constraints~\citep[e.g., see the review by][]{garcia2015comprehensive}. More recent research on \emph{constraint learning} addresses this limitation. \citet{scobee2019maximum} and \citet{stocking2021discretizing} use maximum likelihood estimation to learn constraints from demonstrations with known rewards in tabular and continuous environments, respectively. \citet{anwar2020inverse} propose a more scalable algorithm based on maximum entropy IRL, and extensions to using maximum causal entropy IRL have been explored subsequently~\citep{glazier2021making, mcpherson2021maximum, baert2023maximum}. \citet{papadimitriou2021bayesian} adapt Bayesian IRL to learning constraints via a Lagrangian formulation. However, all of these approaches assume full knowledge of the demonstrations' reward functions. In contrast, \citet{chou2020learning} propose a method that allows for parametric uncertainty about the rewards, but it requires fully known environment dynamics and does not apply to general CMDPs.

\paragraph{Multi-task IRL.}
Multi-task IRL addresses the problem of learning from demonstrations in different but related tasks. Some methods reduce multi-task IRL to multiple single-task IRL problems~\citep{babes2011apprenticeship,choi2012nonparametric}, but they do not fully leverage the similarity between demonstrations. Others treat multi-task IRL as a meta-learning problem, focusing on quickly adapting to new tasks~\citep{dimitrakakis2012bayesian, xu2019learning,yu2019meta,wang2021meta}. \citet{amin2017repeated} study \emph{repeated IRL} where the reward for different tasks is split into a task-specific component and a shared component, similar to the shared constraints in our setting. However, in general, IRL-based methods are challenging to adapt to CMDPs where safety is crucial (see \Cref{sec:irl_limitations}).

\paragraph{Learning constraints for driving behavior.} 
In the domain of autonomous driving, inferring constraints from demonstrations has attracted significant attention due to safety concerns. \citet{rezaee2022not} extend the method by \citet{scobee2019maximum} with a VAE-based gradient descent optimization to learn driving constraints. \citet{liu2022benchmarking} propose a benchmark for constraint learning based on highway driving trajectories, and \citet{gaurav2022learning} evaluate their constraint learning method on a similar dataset. However, all these approaches assume fully known reward functions, which is often unrealistic beyond basic highway driving scenarios.

\paragraph{Bandits \& learning theory.}
Our approach draws inspiration from constraint inference in other fields. For instance, learning constraints has been studied in best-arm identification in linear bandits with known rewards~\citep{lindner2022interactively,camilleri2022active}. Our problem is also related to one-class classification~\citep[e.g.,][]{khan2014one}, the study of random polytopes~\citep[e.g.,][]{baddeley2007random}, and PAC-learning of convex polytopes~\citep[e.g.,][]{gottlieb2018learning}. However, none of these methods directly apply to the structure of our problem.

\section{Experiments}\label{sec:experiments}

\begin{figure*}[t]
\centering
\includegraphics[width=0.9\linewidth]{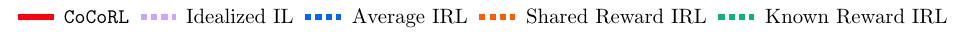} \\
\hspace{-1.5em}
\begin{subfigure}[b]{0.329\linewidth}
    \centering
    \includegraphics[width=1.03\linewidth]{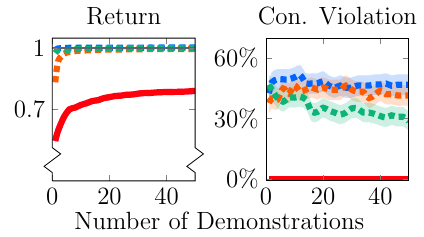}
    \caption{Single environment.}
    \label{subfig:gridworld_exp1}
\end{subfigure}
\begin{subfigure}[b]{0.329\linewidth}
    \centering
    \includegraphics[width=1.03\linewidth]{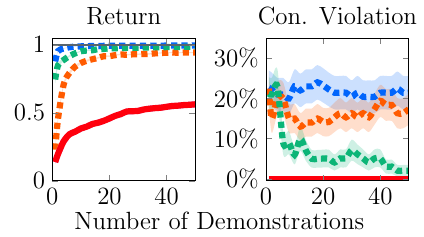}
    \caption{Transfer to new task.}
    \label{subfig:gridworld_exp2}
\end{subfigure}
\begin{subfigure}[b]{0.329\linewidth}
    \centering
    \includegraphics[width=1.03\linewidth]{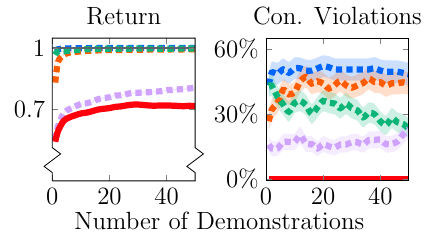}
    \caption{Transfer to new environment.}
    \label{subfig:gridworld_exp3}
\end{subfigure}
\caption{Experimental results in Gridworld environments. We consider three settings: (\subref{subfig:gridworld_exp1}) no constraint transfer, (\subref{subfig:gridworld_exp2}) transferring constraints to new goals in the same grid, and (\subref{subfig:gridworld_exp3}) transferring constraints to a new Gridworld with the same structure but different transition dynamics. For each setting, we measure the normalized policy return (\textbf{higher is better}), and the constraint violation (\textbf{lower is better}). The plots show mean and standard errors over 100 random seeds. 
\AlgNameShort consistently returns safe solutions, outperforming the IRL-based methods that generally perform worse and are unsafe. The IL baseline performs exactly the same as \AlgNameShort with no environment transfer (the lines overlap in plots \subref{subfig:gridworld_exp1} and \subref{subfig:gridworld_exp2}), but it produces unsafe solutions with transfer (\subref{subfig:gridworld_exp3}).
A return greater than 1 indicates a solution that surpasses the best safe policy, implying a constraint violation.
}
\label{fig:gridworld_results}
\end{figure*}

We present experiments in two domains: (1) tabular environments, where we can solve MDPs and CMDPs exactly (\Cref{sec:tabular-experiments}); and (2) a driving simulation, where we investigate the scalability of \AlgNameShort in a more complex environment (\Cref{sec:highway-experiments}). Our experiments assess the safety and performance of \AlgNameShort and compare it to IRL-based baselines (\Cref{sec:irl-baselines}). In \Cref{app:experiment_details}, we provide more details about the experimental setup, and in \Cref{app:additional_results} we present additional experimental results, including an extensive study of \AlgNameShort in one-state CMDPs.
We provide code for our experiments at \githublink.

\subsection{Baselines}\label{sec:irl-baselines}

To the best of our knowledge, \AlgNameShort is the first algorithm that learns constraints from demonstrations with unknown rewards. Thus it is not immediately clear which baselines to compare it to. Nevertheless, we explore several approaches based on IRL and imitation learining (IL) that could serve as natural starting points for solving the constraint learning problem.

\paragraph{IRL baselines.}
We consider the following three ways of applying IRL:
(1) \emph{Average IRL} infers a separate reward function $\hat{\reward}_i$ for each demonstration $\policy_i^*$ and averages the inferred rewards before combining them with an evaluation reward $\evalreward + \sum_i \hat{\reward}_i / \NumDemos$. 
(2) \emph{Shared Reward IRL} parameterizes the inferred rewards as $\hat{\reward}_i + \hat{\cost}$, where $\hat{\cost}$ is shared among all demonstrations. The parameters for both the inferred rewards and the shared constraint penalty are learned simultaneously.
(3) \emph{Known Reward IRL} parameterizes the inferred reward as $\reward_i + \hat{\cost}$, where $\reward_i$ is known, and only a shared constraint penalty is learned. Note that \emph{Known Reward IRL} has full access to the demonstrations' rewards that the other methods do not need. Each of these approaches can be implemented with any standard IRL algorithm. In this section, we present results using Maximum Entropy IRL~\citep{ziebart2008maximum}, but in \Cref{app:additional_results} we also test Maximum Margin IRL~\citep{ng2000algorithms} and obtain qualitatively similar results.

\paragraph{Idealized IL baseline.}
Imitation learning (IL) uses supervised learning to imitate a demonstration \citep{hussein2017imitation}. We can run IL on our set of demonstrations to obtain a set of safe policies. As a baseline, we consider an algorithm that remembers these IL policies and for each evaluation reward $\evalreward$ chooses the best policy from the IL step. We implement this baseline using an idealized form of imitation learning: we directly use the policies that generated the demonstrations. This gives us an upper bound on the possible performance of a real IL method.
For linear rewards and constraints, this algorithm will return the same policies as \AlgNameShort. However it is expensive to run in practice and, as we will see, not robust to changes in the evaluation task or the environment.

For more details about the baselines, see \Cref{app:irl-baselines}.

\subsection{Gridworld Environments}\label{sec:tabular-experiments}

\begin{figure*}
\centering
\includegraphics[width=0.9\linewidth]{plots/gridworld_legend.pdf} \\
\begin{subfigure}[b]{0.329\linewidth}
    \centering
    \includegraphics[width=1.03\linewidth]{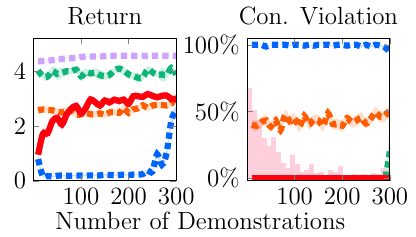}
    \caption{Single environment.}
\end{subfigure}
\begin{subfigure}[b]{0.329\linewidth}
    \centering
    \includegraphics[width=1.03\linewidth]{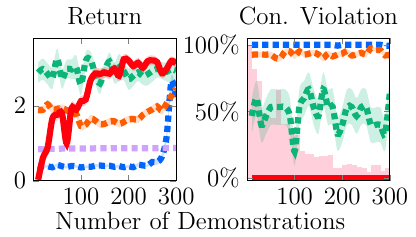}
    \caption{Transfer to new task.}
\end{subfigure}
\begin{subfigure}[b]{0.329\linewidth}
    \centering
    \includegraphics[width=1.03\linewidth]{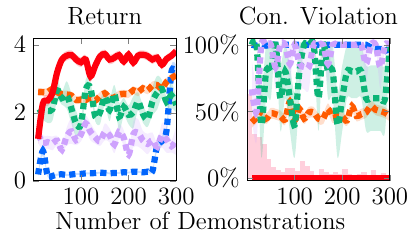}
    \caption{Transfer to new environment.}
\end{subfigure}
\caption{
Return and constraint violation in the \texttt{highway-env} intersection environment.
All plots show mean and standard error over $5$ random seeds with a fixed set of evaluation rewards for each setting.
In all settings, \AlgNameShort consistently returns safe policies, as indicated by the low constraint violation values. However, there are instances where \AlgNameShort falls back to providing a default safe solution when the policy optimizer fails to find a feasible solution within the safe set $\safeset$. The frequency of falling back to the default solution is shown by the bars in the constraint violation plots (\legendCocorlAbstained). In contrast, the IRL method often yields unsafe solutions. IL outperforms \AlgNameShort because we implement an idealized version with perfect imitation. However, for task transfer IL performs much worse than \AlgNameShort and for environment transfer it produces unsafe solutions.
}
\label{fig:highway_results}
\end{figure*}

We first consider a set of Gridworld environments which we can, conveniently, solve using linear programming~\citep{altman1999constrained}. 
The Gridworlds have goal cells and ``limited'' cells. When the agent reaches a goal cell, it receives a reward; however, a set of constraints restrict how often the agent can visit the limited cells.

To introduce stochasticity, there is a fixed probability $p$ that the agent executes a random action instead of the intended one. The positions of the goal and limited cells are sampled uniformly at random. The reward and cost values are sampled uniformly from the interval $[0, 1]$. The thresholds are also sampled uniformly, while we ensure feasibility via rejection sampling.

We perform three experiments to evaluate the transfer of learned constraints. First, we evaluate the learned constraints in the same environment and with the same reward distribution that they were learned from. Second, we change the reward distribution by sampling a new set of potential goals. Third, we use a stochastic Gridworld ($p=0.2$) during training and a deterministic one ($p=0$) during evaluation.

\Cref{fig:gridworld_results} presents the results of the experiments in $10 \times 10$ Gridworlds with $20$ goals and $10$ limited cells. \AlgNameShort consistently returns safe policies and converges to the best safe solution as more demonstrations are provided. In contrast, none of the IRL-based methods learns safe policies. The IL baseline performs exactly the same as \AlgNameShort without transfer and when transferring to a new task,\footnote{
Which is expected in tabular environments, as discussed in \Cref{app:irl-baselines}.
}
but it produces unsafe solutions if the environment dynamics change.
In the on-distribution evaluation shown in \Cref{subfig:gridworld_exp1}, the IRL methods come closest to the desired safe policy but still cause many constraint violations.

\subsection{Driving Environment}\label{sec:highway-experiments}

To explore a more practical and interesting environment, we consider \texttt{highway-env}, a 2D driving simulator~\citep{highway-env}. We focus on a four-way intersection, as depicted in \Cref{fig:headline_figure}. The agent drives towards the intersection and can turn left, right, or go straight.
The environment has a continuous state space, but a discrete high-level action space: the agent can accelerate, decelerate, and make turns at the intersection. There are three possible goals for the agent: turning left, turning right, or going straight. Additionally, the reward functions contain different driving preferences related to velocity and heading angle. Rewards and constraints are defined as linear functions of a state-feature vector that includes goal indicators, vehicle position, velocity, heading, and indicators for unsafe events including crashing and leaving the street.

For constrained policy optimization, we use a constrained cross-entropy method (CEM) based on \citet{wen2018constrained} to optimize a parametric driving controller (see \Cref{app:implementation_details} for details). We collect synthetic demonstrations that are optimized for different reward functions under a shared set of constraints.

Similar to the previous experiments, we consider three settings: (1) evaluation on the same task distribution and environment; (2) transfer to a new task; and (3) transfer to a modified environment. To transfer to a new task, we learn the constraints from trajectories only involving left and right turns at the intersection, but evaluate them on a reward function that rewards going straight (the situation illustrated in \Cref{fig:headline_figure}). For the transfer to a modified environment, we infer constraints in an environment with very defensive drivers and evaluate them with very aggressive drivers. \Cref{app:experiment_details} contains additional details on the parameterization of the drivers.

Empirically, we observe that the CEM occasionally fails to find a feasible policy when the safe set $\safeset$ is small. This is likely due to the CEM relying heavily on random exploration to discover an initial feasible solution. In situations where we cannot find a solution in $\safeset$, we return a ``default'' safe controller that remains safe but achieves a return of 0 because it just stands still before the intersection. Importantly, we empirically confirm that whenever we find a solution in $\safeset$, it is truly safe (i.e., in $\truesafeset$).

\Cref{fig:highway_results} presents results from the driving experiments. \AlgNameShort consistently guarantees safety, and with a large enough number of demonstrations ($\sim 200$), \AlgNameShort returns high return policies. In contrast, IRL produces policies with decent performance in terms of return but they tend to be unsafe.

Although the IRL solutions usually avoid crashes (which also result in low rewards), they disregard constraints that are in conflict with the reward function, such as the speed limits and keeping distance to other vehicles. IRL results in substantially more frequent constraint violations when trying to transfer the learned constraint penalty to a new reward. This is likely because the magnitude of the constraint penalty is no longer correct.

The IL baseline performs better than \AlgNameShort when evaluated in-distribution because our idealized implementation avoid all issues related to optimizing for a policy. However, IL clearly fails in both transfer settings. IL can not learn a policy for a new task that is not in the demonstrations, and it might learn policies via imitation that are no longer safe if the environment dynamics change.
\AlgNameShort does not suffer from this issue, and transferring the inferred constraints to a new reward still results in safe and well-performing policies.

\section{Conclusion}

We introduced \AlgNameShort to infer shared constraints from demonstrations with unknown rewards. Theoretical and empirical results show that \AlgNameShort guarantees safety and achieves strong performance, even when transferring constraints to new tasks or environments.

\AlgNameShort's main limitation is assuming safe demonstrations and access to a feature representation in non-tabular environments. Learning from potentially unsafe demonstrations and learning features to represent constraints are exciting directions for future work.

Learning constraints has the potential to make RL more sample-efficient and safe. Importantly, \AlgNameShort can use unlabeled demonstrations, which are often easier to obtain than labelled demonstrations. Thus, \AlgNameShort can make learning constraints more practical in a range of applications.

\section*{Acknowledgements}

This project was supported by
the Microsoft Swiss Joint Research Center (Swiss JRC),
the Swiss National Science Foundation (SNSF) under NCCR Automation, grant agreement 51NF40 180545,
the European Research Council (ERC) under the European Union’s Horizon 2020 research and innovation programme grant agreement No 815943,
the Vienna Science and Technology Fund (WWTF) [10.47379/ICT20058],
the Open Philanthropy AI Fellowship,
and the Vitalik Buterin PhD Fellowship.
We thank Yarden As for valuable discussions throughout the project.

\newcommand{\ICML}{Proceedings of International Conference on Machine Learning (ICML)}
\newcommand{\RSS}{Proceedings of Robotics: Science and Systems (RSS)}
\newcommand{\NeurIPS}{Advances in Neural Information Processing Systems}
\newcommand{\IJCAI}{Proceedings of International Joint Conferences on Artificial Intelligence (IJCAI)}
\newcommand{\ICLR}{International Conference on Learning Representations (ICLR)}
\newcommand{\CoRL}{Conference on Robot Learning (CoRL)}
\newcommand{\UAI}{Uncertainty in Artificial Intelligence (UAI)}
\newcommand{\AAAI}{AAAI Conference on Artificial Intelligence}
\newcommand{\COLT}{Conference on Learning Theory (COLT)}
\newcommand{\AISTATS}{International Conference on Artificial Intelligence and Statistics (AISTATS)}
\newcommand{\TMLR}{Transactions on Machine Learning Research (TMLR)}

\bibliographystyle{abbrvnat}
\bibliography{references}

\section*{Checklist}

 \begin{enumerate}

 \item For all models and algorithms presented, check if you include:
 \begin{enumerate}
   \item A clear description of the mathematical setting, assumptions, algorithm, and/or model. \textbf{Yes}
   \item An analysis of the properties and complexity (time, space, sample size) of any algorithm. \textbf{Yes}
   \item (Optional) Anonymized source code, with specification of all dependencies, including external libraries. \textbf{Yes}
 \end{enumerate}

 \item For any theoretical claim, check if you include:
 \begin{enumerate}
   \item Statements of the full set of assumptions of all theoretical results. \textbf{Yes}
   \item Complete proofs of all theoretical results. \textbf{Yes}
   \item Clear explanations of any assumptions. \textbf{Yes}
 \end{enumerate}

 \item For all figures and tables that present empirical results, check if you include:
 \begin{enumerate}
   \item The code, data, and instructions needed to reproduce the main experimental results (either in the supplemental material or as a URL). \textbf{Yes}
   \item All the training details (e.g., data splits, hyperparameters, how they were chosen). \textbf{Yes}
         \item A clear definition of the specific measure or statistics and error bars (e.g., with respect to the random seed after running experiments multiple times). \textbf{Yes}
         \item A description of the computing infrastructure used. (e.g., type of GPUs, internal cluster, or cloud provider). \textbf{Yes}
 \end{enumerate}

 \item If you are using existing assets (e.g., code, data, models) or curating/releasing new assets, check if you include:
 \begin{enumerate}
   \item Citations of the creator If your work uses existing assets. \textbf{Yes}
   \item The license information of the assets, if applicable. \textbf{Yes}
   \item New assets either in the supplemental material or as a URL, if applicable. \textbf{Yes}
   \item Information about consent from data providers/curators. \textbf{Not Applicable}
   \item Discussion of sensible content if applicable, e.g., personally identifiable information or offensive content. \textbf{Not Applicable}
 \end{enumerate}

 \item If you used crowdsourcing or conducted research with human subjects, check if you include:
 \begin{enumerate}
   \item The full text of instructions given to participants and screenshots. \textbf{Not Applicable}
   \item Descriptions of potential participant risks, with links to Institutional Review Board (IRB) approvals if applicable. \textbf{Not Applicable}
   \item The estimated hourly wage paid to participants and the total amount spent on participant compensation. \textbf{Not Applicable}
 \end{enumerate}

 \end{enumerate}

\clearpage
\appendix\onecolumn

\renewcommand \thepart{}
\renewcommand \partname{}
\doparttoc
\faketableofcontents
\part{Appendix}
\parttoc

\section{Proofs of Theoretical Results}
\label{app:proofs}

This section provides all proofs omitted in the main paper, as well as a few additional results related to estimating feature expectations (\Cref{app:estimated_features_proofs}).

\subsection{Limitations of IRL in CMDPs}

\IRLUnsafe*

\begin{proof}
This follows from the fact that there are CMDPs for which all feasible policies are stochastic. Every MDP, on the other hand, has a deterministic policy that is optimal~\citep{puterman1990markov}.

For example, consider a CMDP with a single state $\StateSpace = \{ \state_1 \}$ and two actions $\ActionSpace = \{ \action_1, \action_2 \}$ with $\TransitionModel(\state_1 | \state_1, \action_1) = \TransitionModel(\state_1, | \state_1, \action_2) = 1$.
We have two cost functions $\cost_1(\state_1, \action_1) = 1, \cost_1(\state_1, \action_2) = 0$ and $\cost_2(\state_1, \action_1) = 0, \cost_2(\state_1, \action_2) = 1$, with thresholds $\threshold_1 = \threshold_2 = \frac12$, and discount factor $\DiscountFactor = 0$. To be feasible, a policy needs to have an occupancy measure $\Occupancy_\policy$ such that
$
\sum_{\state,\action} \Occupancy(\state,\action) \cost_1(\state, \action) \leq \threshold_1
$
and
$
\sum_{\state,\action} \Occupancy_\policy(\state,\action) \cost_2(\state, \action) \leq \threshold_2
$.

In our case with a single state and $\discount = 0$, this simply means that a feasible policy $\policy$ needs to satisfy $\policy(\action_1 | \state_1) \leq \frac12$ and $\policy(\action_2 | \state_1) \leq \frac12$. But this implies that only a uniformly random policy is feasible.

Any IRL algorithm infers a reward function that matches the occupancy measure of the expert policy \citep{abbeel2004apprenticeship}. Consequently, the inferred reward $\reward_{\text{IRL}}$ needs to give both actions the same reward, because otherwise $\policy^*$ would not be optimal for the MDP with $\reward_{\text{IRL}}$. However, that MDP also has a deterministic optimal policy (like any MDP), which is not feasible in the original CMDP.
\end{proof}

\IRLKnownRewardFails*

\begin{proof}
As in the proof of \Cref{thm:irl-unsafe}, we consider a single-state CMDP with $\StateSpace = \{ \state_1 \}$ and $\ActionSpace = \{ \action_1, \action_2 \}$, where $\action_1$.
We have $\cost_1(\state_1, \action_1) = 1, \cost_1(\state_1, \action_2) = 0$ and $\cost_2(\state_1, \action_1) = 0, \cost_2(\state_1, \action_2) = 1$, $\threshold_1 = \threshold_2 = \frac12$, and $\discount = 0$. Again, the only feasible policy uniformly randomizes between $\action_1$ and $\action_2$.

Let $\reward_1(\action_1) = 1$, $\reward_1(\action_2) = 1$, $\reward_2(\action_1) = 0$, $\reward_2(\action_2) = 1$, i.e., the first reward function gives equal reward to both actions and the second reward function gives inequal rewards (for simplicity we write $\reward(\action) = \reward(\state, \action)$). To make the uniformly random policy optimal in $\MDP \cup \{ \reward_1 + \hat{\cost} \}$ and $\MDP \cup \{ \reward_2 + \hat{\cost} \}$, both inferred reward functions $\reward_1 + \hat{\cost}$ and $\reward_2 + \hat{\cost}$ need to give both actions equal rewards. However this is clearly impossible. If $\reward_1(\action_2) + \hat{\cost}(\action_2) - (\reward_1(\action_1) + \hat{\cost}(\action_1)) = 0$, then
\begin{align*}
& \reward_2(\action_2) + \hat{\cost}(\action_2) - (\reward_2(\action_1) + \hat{\cost}(\action_1)) \\
= &\reward_2(\action_2) - \reward_1(\action_2) + \reward_1(\action_2) + \hat{\cost}(\action_2) - (\reward_2(\action_1) - \reward_1(\action_1) + \reward_1(\action_1) + \hat{\cost}(\action_1)) \\
= &\underbrace{\reward_2(\action_2) - \reward_1(\action_2)}_{= 0} - \underbrace{(\reward_2(\action_1) - \reward_1(\action_1))}_{= -1} + \underbrace{\reward_1(\action_2) + \hat{\cost}(\action_2) - (\reward_1(\action_1) + \hat{\cost}(\action_1))}_{= 0}
= 1 \neq 0.
\end{align*}
\end{proof}

\subsection{Safety Guarantees}

\TrueSafeSetConvex*

\begin{proof}
If $\policy_1, \policy_2 \in \truesafeset$, then we have $\CumulativeCost_j(\policy_1) \leq \threshold_j$ and $\CumulativeCost_j(\policy_2) \leq \threshold_j$ for any $j$. Further, also for any $j$, we have
$
\CumulativeCost_j(\bar{\policy}_{12}) = \phi_j^T \FeatureFunction(\bar{\policy}_{12})
= \lambda \phi_j^T \FeatureFunction(\policy_1) + (1-\lambda) \phi_j^T \FeatureFunction(\policy_2) \leq \threshold_j
$. Thus, $\bar{\policy}_{12} \in \truesafeset$.
\end{proof}

\SafeSetIsSafe*

\begin{proof}
Our demonstrations $\policy_1^*, \dots, \policy_\NumDemos^*$ are all in the true safe set $\truesafeset$, and any $\policy \in \safeset$ is a convex combination of them. Given \Cref{lem:true-safe-set-convex}, we have $\safeset \subseteq \truesafeset$.
\end{proof}

\EstimatedCMDP*

\begin{proof}
The convex hull of a set of points is a convex polyhedron, i.e., it is the solution of a set of linear equations~\citep[see, e.g., Theorem 2.9 in][]{pulleyblank1983polyhedral}.

Hence, $\safeset$ is a convex polyhedron in the feature space defined by $\FeatureFunction$, i.e., we can find $A \in \reals^{p \times d}, \vb \in \reals^p$ such that
\[
\safeset = \{ x \in \reals^d | A x \leq \vb \}.
\]
To construct $A$ and $\vb$, we need to find the facets of $\safeset$, the convex hull of a set of points $\demonstrations$. This is a standard problem in polyhedral combinatorics~\citep[e.g., see][]{pulleyblank1983polyhedral}.

We can now define $p$ linear constraint functions $\hat{\cost}_j(\state, \action) = A_j \FeatureFunction(\state, \action)$ and thresholds $\hat{\threshold}_j = b_j$, where $A_j$ is the $j$-th row of $A$ and $b_j$ is the $j$-th component of $\vb$ (for $j = 1, \dots, p$). Then, solving the CMDP
$
(\StateSpace, \ActionSpace, \TransitionModel, \initdist, \DiscountFactor, \evalreward, \{ \hat{\cost}_j \}_{j=1}^p \{ \hat{\threshold}_j \}_{j=1}^p)
$
corresponds to solving
\[
\policy^* \in \argmax_{\hat{\CumulativeCost}_1(\policy) \leq \hat{\threshold}_1, \dots, \hat{\CumulativeCost}_p(\policy) \leq \hat{\threshold}_p} \PolicyReturn_{\evalreward}(\policy) .
\]
\looseness -1
For these linear cost functions, we can rewrite the constraints using the discounted feature expectations as $\CumulativeCost_i(\policy) = A_i \FeatureFunction(\policy)$. Hence, $\policy$ satisfying the constraints $\hat{\CumulativeCost}_1(\policy) \leq \hat{\threshold}_1, \dots, \hat{\CumulativeCost}_p(\policy) \leq \hat{\threshold}_p$ is equivalent to $\policy \in \safeset$.
\end{proof}

\SafeSetOptimal*

\begin{proof}
Using \Cref{thm:estimated-cmdp}, we can construct a set of cost functions $\{ \cost_j \}_{j=1}^\NumConst$ and thresholds $\{ \threshold_j \}_{j=1}^\NumConst$ for which $\policy \in \safeset$ is equivalent to $\CumulativeCost_1(\policy) \leq \threshold_1, \dots, \CumulativeCost_\NumConst(\policy) \leq \threshold_\NumConst$. Hence, in a CMDP with these cost functions and thresholds, any policy $\policy \notin \safeset$ is not feasible.
Because $\policy_1^*, \dots, \policy_\NumDemos^* \in \safeset$ by construction, each $\policy^*_i$ is optimal in the CMDP $(\StateSpace, \ActionSpace, \TransitionModel, \initdist, \DiscountFactor, \reward_i, \{ \cost_j \}_{j=1}^\NumConst, \{ \threshold_j \}_{j=1}^\NumConst)$.
\end{proof}

\subsection{Optimality Guarantees}

\begin{lemma}\label{thm:bounded-features-rewards}
For any policy $\policy$, if $\reward: \StateSpace \times \ActionSpace \to [0, 1]$ and $\FeatureFunction: \StateSpace \times \ActionSpace \to [0, 1]^d$, we can bound
\begin{align*}
\forall i:~ \FeatureFunction(\policy)_i &\leq \frac{1}{1 - \DiscountFactor}
&\| \FeatureFunction(\policy) \|_2 \leq \frac{\sqrt{d}}{1-\DiscountFactor} \\
\PolicyReturn_\reward(\policy) &\leq \frac{d}{1-\DiscountFactor}
&\regret(\policy, \safeset) \leq \frac{2d}{1-\DiscountFactor}
\end{align*}
\end{lemma}

\begin{proof}
First, we have for any component of the feature expectations
\[
\FeatureFunction_i(\policy) = \Expectation [ \sum_{t=0}^\infty \DiscountFactor^t \FeatureFunction(\state_t, \action_t) ] \leq \sum_{t=0}^\infty \DiscountFactor^t = \frac{1}{1-\DiscountFactor},
\]
using the limit of the \emph{geometric series}. This immediately gives
\[
\| \FeatureFunction(\policy) \|_2 = \sqrt{ \sum_{i=1}^d \FeatureFunction(\policy)_i^2 } \leq \frac{\sqrt{d}}{1-\DiscountFactor}.
\]

Similarly bounded rewards $\reward(\state, \action) \leq 1$ imply $\| \theta \|_2 \leq \sqrt{d}$. Together, we can bound the returns using the Cauchy-Schwartz inequality
\begin{align*}
\PolicyReturn_\reward(\policy) = \| \theta^T \FeatureFunction(\policy) \|_2 
\leq \| \theta \|_2 \cdot \| \FeatureFunction(\policy) \|_2
\leq \frac{d}{1-\DiscountFactor} .
\end{align*}

Further, for any two policies $\policy_1, \policy_2$, we have $\PolicyReturn_\reward(\policy_1) - \PolicyReturn_\reward(\policy_2) \leq \frac{2d}{1-\DiscountFactor}$, which implies the same for the regret $\regret(\reward, \safeset) \leq \frac{2d}{1-\DiscountFactor}$.
\end{proof}

\ConvergenceNoiseFree*

\begin{proof}
Consider the distribution over optimal policies $P(\FeatureFunction(\policy^*))$ induced by $P(\reward)$.
For each reward $\reward$, there is a vertex of the true safe set that is optimal and thus is in the support of $P(\FeatureFunction(\policy^*))$.
\footnote{Technically, there are degenerate cases where $P(\reward)$ is only supported on reward functions that are orthogonal to the constraint boundaries and we never see demonstrations at the vertices of the true safe set. This is an artifact of the relatively unnatural assumption of noise-free demonstrations. We can avoid such degenerate cases by mild assumptions: either assuming some minimal noise in $P(\reward)$ or assuming the algorithm generating the demonstrations has non-zero probability for all policies optimal for a given reward.}

We can distinguish two cases. \case{1}: we have seen a vertex corresponding to an optimal policy for $\reward$ in the first $\NumDemos$ demonstrations; and, \case{2}: we have not. In \case{1}, we incur $0$ regret, and only in \case{2} we can incur regret greater than $0$.

So, we can bound the probability of incurring regret by the probability of \case{2}:
\begin{align*}
P(\regret(\reward, \safeset_\NumDemos) > 0) \leq \sum_{\text{vertex } \vv} (1-P(\vv))^\NumDemos P(\vv) .
\end{align*}

To ensure $P(\regret(\reward, \safeset_\NumDemos) > 0) \leq \delta$, it is sufficient to ensure that each term of the sum satisfies $(1-P(\vv))^\NumDemos P(\vv) \leq \delta / N_v$ where $N_v$ is the number of vertices. For terms with $P(\vv) \leq \delta / N_v$ this is true for all $\NumDemos$. For the remaining terms with $P(\vv) > \delta / N_v$, we can write
\[
(1 - P(\vv))^\NumDemos P(\vv) \leq (1 - \delta / N_v)^\NumDemos .
\]
So, it is sufficient to ensure
$
(1 - \delta / N_v)^\NumDemos \leq \delta / N_v
$,
which is satisfied once
\[
\NumDemos > \frac{\log(\delta / N_v)}{\log(1 - \delta/N_v)} .
\]
By replacing $N_v$ with a suitable upper bound on the number of vertices, we arrive at the first result.

By \Cref{thm:bounded-features-rewards}, the maximum regret is upper-bounded by $\frac{2 d}{1-\DiscountFactor}$, and, we can decompose the regret as
\begin{align*}
\Expectation_\reward[\regret(\reward, \safeset_k)] \leq \frac{2d}{1-\DiscountFactor} \sum_{\text{vertex } \vv} (1-P(\vv))^\NumDemos P(\vv)
\leq \frac{2d}{1-\DiscountFactor} \sum_{\vv: P(\vv) > 0} (1-P(\vv))^\NumDemos .
\end{align*}
Because $P(\vv)$ is a fixed distribution induced by $P(\reward)$, the r.h.s.\ converges to $0$ as $\NumDemos\to\infty$.
\end{proof}

\begin{lemma}\label{lem:boltzmann-bound}
Under \Cref{ass:boltzmann-noise}, we have for any reward function $\reward$ and any feasible policy $\policy \in \truesafeset$
\[
P(\policy) \geq \exp(- \beta d / (1-\DiscountFactor)) .
\]
\end{lemma}

\begin{proof}
We have $\beta > 0$ and $\PolicyReturn_\reward(\policy) \geq 0$, which implies $\exp(\beta \PolicyReturn_\reward(\policy)) \geq 1$, and
\begin{align*}
P(\policy | \reward) = \frac{\exp(\beta \PolicyReturn_\reward(\policy)}{Z(\theta)}
\geq \frac{1}{Z(\theta)} .
\end{align*}

By \Cref{thm:bounded-features-rewards}, we have
$
\PolicyReturn_\reward(\policy) \leq d / (1 - \DiscountFactor)
$.
Therefore,
\[
Z(\theta) = \int_{\policy\in\truesafeset} \exp(\beta\theta^T\FeatureFunction(\policy)) \dif\policy \leq \exp(\beta d / (1-\DiscountFactor)) ,
\]
and
\[
P(\policy | \reward) \geq \exp(- \beta d^2 / (1-\DiscountFactor)) .
\]
\end{proof}

\ConvergenceBoltzmann*

\begin{proof}
We can upper-bound the probability of having non-zero regret similar to the noise free case by
\begin{align*}
P(\regret(\reward, \safeset_\NumDemos) > 0) \leq \sum_{\text{vertex } \vv} (1-P(\vv))^\NumDemos P(\vv) .
\end{align*}
Under \Cref{ass:boltzmann-noise}, we can use \Cref{lem:boltzmann-bound}, to obtain
\[
P(\vv) \geq \exp(- \beta d / (1-\DiscountFactor)) \coloneqq \Delta .
\]
Importantly, $0 < \Delta < 1$ is a constant, and we have
\[
P(\regret(\reward, \safeset_\NumDemos) > 0)
\leq \sum_{\text{vertex } \vv} (1 - \Delta)^\NumDemos .
\]
To ensure $P(\regret(\reward, \safeset_\NumDemos) > 0) \leq \delta$, it is sufficient to ensure $(1 - \Delta)^\NumDemos \leq \delta / N_v$, where $N_v$ is the number of vertices of the true safe set. This is true once
\[
\NumDemos \geq \frac{\log(\delta / N_v)}{\log(1 - \exp(-\beta d / (1-\DiscountFactor)))} ,
\]
where we can again replace $N_v$ by a suitable upper bound.

Bounding the maximum regret via \Cref{thm:bounded-features-rewards}, we get
\begin{align*}
\Expectation_\reward[\regret(\reward, \safeset_\NumDemos)] \leq \frac{2d}{1-\DiscountFactor} \sum_{\text{vertex } \vv} (1 - \Delta)^\NumDemos ,
\end{align*}
which converges to $0$ as $\NumDemos \to \infty$.
\end{proof}

\subsection{Estimating Feature Expectations}
\label{app:estimated_features_proofs}

\begin{lemma}\label{thm:feature_expectation_concentration}
Let $\policy$ be a policy with true feature expectation $\FeatureFunction(\policy)$. We estimate the feature expectation using $\NumTraj$ trajectories $\tau_i$ collected by rolling out $\policy$ in the environment: $\hat{\FeatureFunction}(\policy) = \frac{1}{\NumTraj} \sum_{i=1}^{\NumTraj} \FeatureFunction(\tau_i)$. This estimate is unbiased, i.e., $\Expectation[\hat{\FeatureFunction}(\policy)] = \FeatureFunction(\policy)$. Further let $\phi\in\reals^d$ be a vector, such that, $0 \leq \phi^T \FeatureFunction(\state, \action) \leq 1$ for any state $\state$ and action $\action$. Then, we have for any $\epsilon > 0$
\begin{align*}
P( \phi^T \hat{\FeatureFunction}(\policy) - \phi^T \FeatureFunction(\policy) \geq \epsilon ) &\leq \exp( - 2 \epsilon^2 \NumTraj (1 - \DiscountFactor) / d ) , \\
P( \phi^T \FeatureFunction(\policy) - \phi^T \hat{\FeatureFunction}(\policy) \geq \epsilon ) &\leq \exp( - 2 \epsilon^2 \NumTraj (1 - \DiscountFactor) / d ) .
\end{align*}
\end{lemma}

\begin{proof}
Because the expectation is linear, the estimate is unbiased, and, $\Expectation[\phi^T \hat{\FeatureFunction}(\policy)] = \phi^T \FeatureFunction(\policy)$.

For any trajectory $\tau$ and any $1 \leq i \leq d$, we have $0 \leq \FeatureFunction_i(\tau) \leq 1/(1-\DiscountFactor)$, and, consequently, $0 \leq \phi^T \FeatureFunction(\tau) \leq d/(1-\DiscountFactor)$, analogously to \Cref{thm:bounded-features-rewards}). Thus, $\phi^T \hat{\FeatureFunction}(\policy)$ satisfies the bounded difference property. In particular, by changing one of the observed trajectories, the value of $\phi^T \hat{\FeatureFunction}(\policy)$ can change at most by $d / (\NumTraj(1-\DiscountFactor))$.
Hence, by McDiarmid's inequality, we have for any $\epsilon > 0$
\[
P( \theta^T \hat{\FeatureFunction}(\policy) - \Expectation[\theta^T \hat{\FeatureFunction}(\policy)] \geq \epsilon)  \leq \exp\left( -\frac{2 \epsilon^2 }{d/(\NumTraj(1-\DiscountFactor))} \right) ,
\]
and
\[
P( \theta^T \hat{\FeatureFunction}(\policy) - \theta^T \FeatureFunction(\policy) \geq \epsilon) \leq \exp\left( -\frac{2 \epsilon^2 \NumTraj(1-\DiscountFactor) }{d} \right) .
\]

We can bound $P( \phi^T \FeatureFunction(\policy) - \phi^T \hat{\FeatureFunction}(\policy) \geq \epsilon )$ analogously.
\end{proof}

\EstimatedFeaturesSafety*

\begin{proof}
Let us define the ``good'' event that for policy $\policy_i^*$ we accurately estimate the value of the cost function $\cost_j$ and denote it by $\mathcal{E}_{ij} = \{ \phi_j^T \hat{\FeatureFunction}(\policy_i^*) - \phi_j^T \FeatureFunction(\policy_i^*) \leq \epsilon \}$, where $\phi_j$ parameterizes $\cost_j$. Conditioned on $\mathcal{E}_{ij}$ for all $i$ and $j$, we have
\begin{align*}
P(\max_j (\CumulativeCost_j(\policy) - \threshold_j) &> \epsilon | \policy \in \hat{\safeset}, \{ \mathcal{E}_{ij} \} )
\leq \sum_j P((\CumulativeCost_j(\policy) - \threshold_j) > \epsilon | \policy \in \hat{\safeset}, \{ \mathcal{E}_{ij} \} ) \\
&= \sum_j P( \phi_j^T \sum_l \lambda_l \hat{\FeatureFunction}(\policy_l^*) > \threshold_j + \epsilon | \{ \mathcal{E}_{ij} \} ) = 0 .
\end{align*}
Hence, we can bound the probability of having an unsafe policy by the probability of the ``bad'' event, which we can bound by \Cref{thm:feature_expectation_concentration} to obtain
\begin{align*}
P(\max_j (\CumulativeCost_j(\policy) - \threshold_j) > \epsilon | \policy \in \hat{\safeset} )
\leq \sum_{ij} P(\text{not }\mathcal{E}_{ij})
\leq \NumConst \cdot \NumDemos \cdot \exp( - 2 \epsilon^2 \NumTraj (1 - \DiscountFactor) / d ) .
\end{align*}
So, to ensure $P(\max_j (J_j(\policy) - \threshold_j) > \epsilon | \policy \in \hat{\safeset} \} ) \leq \delta$, it is sufficient to ensure
\[
\exp( - 2 \epsilon^2 \NumTraj (1 - \DiscountFactor) / d ) \leq \frac{\delta}{\NumConst \NumDemos} ,
\]
which is true if we collect at least
\[
\NumTraj > \frac{d \log(\NumConst \NumDemos / \delta)}{2 \epsilon^2 (1 - \DiscountFactor)}
\]
trajectories for each policy.
\end{proof}

\begin{theorem}[Convergence, noise-free, estimated feature expectations]\label{thm:convergence-noise-free-estimated-features}
Under \Cref{ass:exact-optimality}, if we estimate the feature expectations of at least
$
\NumDemos > \log(\delta / (2 f_v(d, \NumConst))) / \log(1 - \delta / (2 f_v(d, \NumConst)))
$
expert policies using at least
$
\NumTraj > d \log(2 f_v(d, \NumConst) / \delta) / (2 \epsilon^2 (1 - \DiscountFactor))
$
samples each, we have $P(\regret(\reward, \safeset_\NumDemos) > \epsilon) \leq \delta$, where $f_v(d, \NumConst)$ is an upper bound on the number of vertices of the true safe set. In particular, we have
$
\lim_{\min(\NumDemos, \NumTraj)\to\infty} \Expectation_\reward \left[ \regret(\reward, \safeset_\NumDemos) \right] = 0
$.
\end{theorem}

\begin{proof}
To reason about the possibility of a large estimation error for the feature expectations, we use \Cref{thm:feature_expectation_concentration} to lower bound the probability of the good event $\mathcal{E}_{\vv} = \{ \theta^T \vv - \theta^T \hat{\FeatureFunction}(\policy_i^*) \leq \epsilon \}$, for each vertex $\vv$ of the true safe set.

For an evaluation reward $\evalreward$, we can incur regret in one of three cases:\\
\noindent \case{1}: we have \emph{not} seen a vertex corresponding to an optimal policy for $\evalreward$. \\
\noindent \case{2}: we have seen an optimal vertex but our estimation of that vertex is \emph{bad}. \\
\noindent \case{3}: we have seen an optimal vertex and our estimation of that vertex is \emph{good}.

We show that the probability of \case{1} shrinks as we increase $\NumDemos$, the probability of \case{2} shrinks as we increase $\NumTraj$, and the regret that we can incur in \case{3} is small.

\case{1}: This case is independent of the estimated feature expectations. Similar to \Cref{thm:convergence-noise-free}, we can bound its probability for each vertex $\vv$ by $(1-P(\vv))^\NumDemos P(\vv)$.

\case{2}: In case of the bad event, i.e., the complement of $\mathcal{E}_\vv$, we can upper-bound the probability of incurring high regret using \Cref{thm:feature_expectation_concentration}, to obtain
\[
P( \theta^T \vv - \theta^T \hat{\FeatureFunction}(\policy_i^*) > \epsilon | \text{not }\mathcal{E}_\vv ) P(\text{not }\mathcal{E}_\vv)
\leq P(\text{not }\mathcal{E}_\vv)
\leq \exp( - 2 \epsilon^2 \NumTraj (1 - \DiscountFactor) / d ) .
\]

\case{3}: Under the good event $\mathcal{E}_\vv$, we automatically have regret less than $\epsilon$, i.e.,
\[
P( \theta^T \vv - \theta^T \hat{\FeatureFunction}(\policy_i^*) > \epsilon | \mathcal{E}_\vv ) P(\mathcal{E}_\vv) = 0.
\]

Using these three cases, we can now bound the probability of having regret greater than $\epsilon$ as

\begin{align*}
P(\regret(\reward, \safeset_k) > \epsilon) \leq \sum_{\text{vertex }\vv} \Bigl(
\underbrace{(1-P(\vv))^k P(\vv)}_{\case{1}}
+ \underbrace{\exp( - 2 \epsilon^2 \NumTraj (1 - \DiscountFactor) / d )}_{\case{2}}
+ \underbrace{0}_{\case{3}}
\Bigr).
\end{align*}

To ensure that $P(\regret(\reward, \safeset_k) > \epsilon) \leq \delta$, it is sufficient to ensure each term of the sum is less than $\delta / N_v$ where $N_v$ is the number of vertices. We have two terms inside the sum, so it is sufficient to ensure either term is less than $\delta / (2 N_v)$.

\case{1}: We argue analogously to \Cref{thm:convergence-noise-free}. For terms with $P(\vv) \leq \delta / (2 N_v)$ it is true for all $\NumDemos$. For the remaining terms with $P(\vv) > \delta / (2 N_v)$, we can use
\[
(1 - P(\vv))^\NumDemos P(\vv) \leq (1 - \delta / (2 N_v))^\NumDemos
\]
so it is sufficient to ensure
\[
(1 - \delta / (2 N_v))^\NumDemos \leq \delta / (2 N_v)
\]
which is satisfied once
\[
k > \frac{\log(\delta / (2 N_v))}{\log(1 - \delta / (2 N_v))} .
\]
This is our first condition.

\case{2}: To ensure $\exp( - 2 \epsilon^2 \NumTraj (1 - \DiscountFactor) / d ) \leq \delta / (2 N_v)$, we need
\[
\NumTraj > \frac{d \log(2 N_v / \delta)}{2 \epsilon^2 (1 - \DiscountFactor)} .
\]
This is our second condition.

If both conditions hold simultaneously, we have $P(\regret(\reward, \safeset_\NumDemos) > \epsilon) \leq \delta$.

To analyse the regret, let us consider how $\epsilon$ shrinks as a function of $\NumTraj$. The analysis above holds for any $\epsilon$. So, for a given $\NumDemos$ and $\NumTraj$, we need to chose
\[
\epsilon^2 \geq \frac{d \log(2 \NumConst \NumDemos N_v / \delta)}{2 \NumTraj (1 - \DiscountFactor)}
\]
to guarantee $P(\regret(\reward, \safeset_k) > \epsilon) \leq \delta$. Hence, the smallest $\epsilon$ we can choose is
\[
\epsilon_{\min} = \sqrt{\frac{d \log(2 \NumConst \NumDemos N_v / \delta)}{2 \NumTraj (1 - \DiscountFactor)}} .
\]

Using \Cref{thm:bounded-features-rewards} to upper-bound the maximum regret by $\frac{2 d^2}{1-\DiscountFactor}$, we can upper-bound the expected regret by
\begin{align*}
&\Expectation_\reward[\regret(\reward, \safeset_k)]
\leq \frac{2 d^2}{1-\DiscountFactor} P(\regret(\reward, \safeset_k) > \epsilon_{\min}) + \epsilon_{\min} P(\regret(\reward, \safeset_k) \leq \epsilon_{\min}) \\
&\leq \frac{2d}{1-\DiscountFactor} \sum_{\vv: P(\vv) > 0} (1-P(\vv))^\NumDemos + \frac{2d N_v}{1-\DiscountFactor} \exp( - 2 \epsilon_{\min}^2 (1 - \DiscountFactor) / d ) \frac{\NumDemos}{\exp(\NumTraj)}
+ \sqrt{\frac{d \log(2 \NumConst \NumDemos N_v / \delta)}{2 \NumTraj (1 - \DiscountFactor)}} .
\end{align*}

For each term individually, we can see that it converges to $0$ as $\min(\NumDemos, \NumTraj) \to \infty$:
\[
\lim_{\min(\NumDemos, \NumTraj) \to \infty}  \sum_{\vv: P(\vv) > 0} (1-P(\vv))^\NumDemos = \lim_{\NumDemos \to \infty}  \sum_{\vv: P(\vv) > 0} (1-P(\vv))^\NumDemos = 0 ,
\]
\[
\lim_{\min(\NumDemos, \NumTraj) \to \infty} \underbrace{\exp( - 2 \epsilon_{\min}^2 (1 - \DiscountFactor) / d )}_{\leq 1} \frac{\NumDemos}{\exp(\NumTraj)}
\leq \lim_{\min(\NumDemos, \NumTraj) \to \infty} \frac{\NumDemos}{\exp(\NumTraj)} = 0 ,
\]
\[
\lim_{\min(\NumDemos, \NumTraj) \to \infty} \sqrt{\frac{d \log(2 \NumConst \NumDemos N_v / \delta)}{2 \NumTraj (1 - \DiscountFactor)}} = 0 .
\]
Hence,  $\lim_{\min(\NumDemos, \NumTraj) \to \infty} \Expectation_\reward[\regret(\reward, \safeset_k)] = 0$.
\end{proof}

\begin{theorem}[Convergence under Boltzmann noise, estimated feature expectations]
Under \Cref{ass:boltzmann-noise}, if we estimate the feature expectations of at least
$
\NumDemos > \log(\delta / (2 f_v(d, \NumConst))) / (\log(1 - \exp(- \beta d / (1-\DiscountFactor))))
$
expert policies using at least
$
\NumTraj > d \log(2 f_v(d, \NumConst) / \delta) / (2 \epsilon^2 (1 - \DiscountFactor))
$
samples each, we have $P(\regret(\reward, \safeset_\NumDemos) > \epsilon) \leq \delta$, where $f_v(d, \NumConst)$ is an upper bound on the number of vertices of the true safe set. In particular, we have
$
\lim_{\min(\NumDemos, \NumTraj)\to\infty} \Expectation_\reward \left[ \regret(\reward, \safeset_\NumDemos) \right] = 0
$.
\end{theorem}

\begin{proof}
We can decompose the probability of incurring regret greater than $\epsilon$ into the same three cases discussed in the proof of \Cref{thm:convergence-noise-free-estimated-features} to get the upper-bound

\begin{align*}
P(\regret(\reward, \safeset_k) > \epsilon) \leq \sum_{\text{vertex }\vv} \Bigl(
\underbrace{(1-P(\vv))^k P(\vv)}_{\case{1}}
+ \underbrace{\exp( - 2 \epsilon^2 \NumTraj (1 - \DiscountFactor) / d )}_{\case{2}}
+ \underbrace{0}_{\case{3}} .
\Bigr)
\end{align*}

Under \Cref{ass:boltzmann-noise}, we can use \Cref{lem:boltzmann-bound}, to get
\[
P(\vv) \geq \exp(- \beta d / (1-\DiscountFactor)) \coloneqq \Delta .
\]
Combining both bounds, we get
\begin{align*}
P(\regret(\reward, \safeset_\NumDemos) > \epsilon) \leq \sum_{\text{vertex }\vv} \Bigl(
(1-\Delta)^\NumDemos
+ \exp( - 2 \epsilon^2 \NumTraj (1 - \DiscountFactor) / d )
\Bigr) .
\end{align*}

To ensure $P(\regret(\reward, \safeset_\NumDemos) > \epsilon) \leq \delta$ it is sufficient to ensure both $(1 - \Delta)^\NumDemos \leq \delta / (2 N_v)$ and $\exp( - 2 \epsilon^2 \NumTraj (1 - \DiscountFactor) / d  \leq \delta / (2 N_v)$ where $N_v$ is the number of vertices of the true safe set.
This is true once
\[
\NumDemos > \frac{\log(\delta / (2 N_v))}{\log(1 - \Delta)}
= \frac{\log(\delta / (2 N_v))}{\log(1 - \exp(- \beta d / (1-\DiscountFactor)))} ,
\]
and
\[
\NumTraj > \frac{d \log(2 N_v / \delta)}{2 \epsilon^2 (1 - \DiscountFactor)} ,
\]
where can again replace $N_v$ by a suitable upper bound.

For the regret, we get asymptotic optimality with the same argument from the proof of \Cref{thm:convergence-noise-free-estimated-features}, replacing $P(\vv)$ with the constant $\Delta$.
\end{proof}

\section{Constructing a Conservative Safe Set from Estimated Feature Expectations}
\label{app:conservative_safe_set_from_estimated_features}

\Cref{thm:estimated-features-safety} shows that we can guarantee $\epsilon$-safety even when estimating the feature expectations of the demonstrations. However, sometimes we need to ensure exact safety. In this section, we discuss how we can use confidence intervals around estimated feature expectations to construct a conservative safe set, i.e., a set $\hat{\safeset} \subset \safeset$ that still guarantees \emph{exact} safety with high probability (w.h.p.).

\newcommand{\ConfidenceSet}{\mathcal{C}}

Suppose we have $\NumDemos$ confidence sets $\ConfidenceSet_i \subseteq \reals^d$, such that we know $\FeatureFunction(\policy_i^*) \in \ConfidenceSet_i$  (w.h.p.). Now we want to construct a conservative convex hull, i.e., a set of points in which any point certainly in the convex hull of $\FeatureFunction(\policy_1^*), \dots, \FeatureFunction(\policy_\NumDemos^*)$. In other words, we want to construct the intersection of all possible convex hulls with points from $\ConfidenceSet_1, \dots, \ConfidenceSet_\NumDemos$:
\[
\hat{S} = \bigcap\limits_{x_1, \dots, x_\NumDemos \in \ConfidenceSet_1 \times \dots \times \ConfidenceSet_\NumDemos} \conv(x_1, \dots, x_\NumDemos) .
\]

This set is sometimes called the \emph{guaranteed hull}~\citep{sember2011guarantees}, or the \emph{interior convex hull}~\citep{edalat2005computability}. Efficient algorithms are known to construct guaranteed hulls in $2$ or $3$ dimensions if $\ConfidenceSet$ has a simple shape, such as a cube or a disk \citep[e.g., see][]{sember2011guarantees}. However, we need to construct a guaranteed hull in $d$ dimensions.

\citet{edalat2001convex} propose to reduce the problem of constructing guaranteed hull when $\ConfidenceSet_i$ are (hyper-)rectangles to computing the intersection of convex hulls constructed from combinations of the corners of the rectangle. In $d \leq 3$ dimensions, this is possible with $\bigO(\NumDemos \log \NumDemos)$ complexity, but in higher dimensions is requires $\bigO(\NumDemos^{\lfloor d/2 \rfloor})$. If we use a concentration inequality for each coordinate independently, e.g., by using \Cref{thm:feature_expectation_concentration} with unit basis vectors, then the $\ConfidenceSet_i$'s are rectangles, and we can use this approach. However, because we have to construct many convex hulls, it can get expensive as $\NumDemos$ and $d$ grow.

\Cref{fig:interior_hull_illustration} illustrates the guaranteed hull growing as the confidence regions shrink. The guaranteed hull is a conservative estimate and can be empty if the $\ConfidenceSet_i$'s are too large. Hence, $\hat{\safeset}$ might be too conservative sometimes, especially if we cannot observe multiple trajectories from the same policy. In such cases, we need additional assumptions about the environment (e.g., it is deterministic) or the demonstrations (e.g., every trajectory is safe) to guarantee safety in practice. Still, using a guaranteed hull as a conservative safe set is a promising alternative to relying on point estimates of the feature expectations.

\begin{figure}
    \centering
    \fbox{\includegraphics[width=0.22\linewidth]{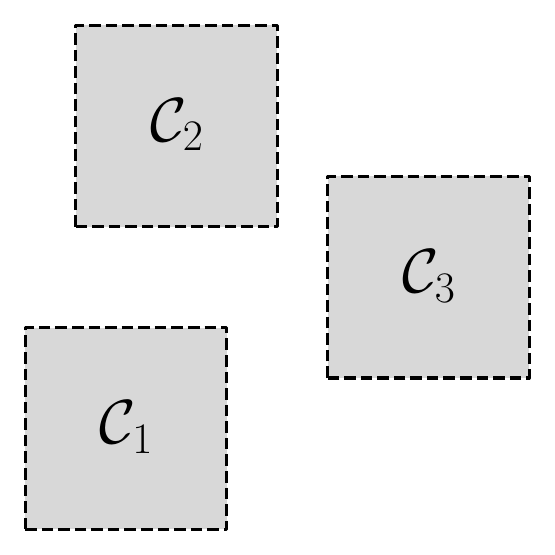}}
    \hfill
    \fbox{\includegraphics[width=0.22\linewidth]{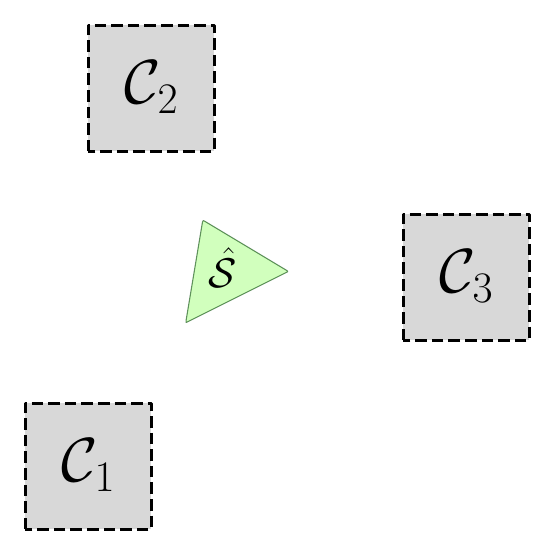}}
    \hfill
    \fbox{\includegraphics[width=0.22\linewidth]{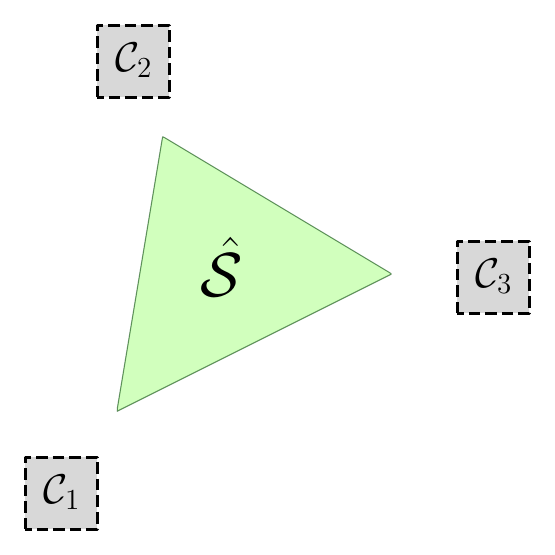}}
    \hfill
    \fbox{\includegraphics[width=0.22\linewidth]{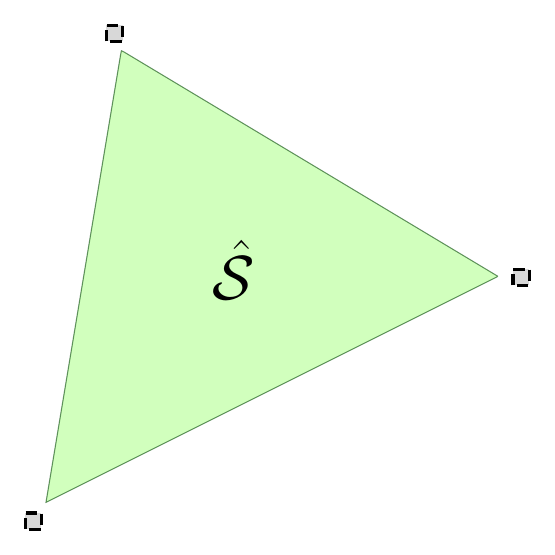}}
    \caption{
    Illustration of the \emph{guaranteed hull} in 2D. The gray areas depict confidence sets $\ConfidenceSet_1, \ConfidenceSet_2, \ConfidenceSet_3$, and the green area is the conservative safe set $\hat{\safeset}$ constructed using the guaranteed hull. For large confidence sets, the guaranteed hull can be empty (left); for small confidence sets, it approaches the safe set constructed using the exact feature expectations (right).}
    \label{fig:interior_hull_illustration}
\end{figure}

\section{Excluding Guaranteed Unsafe Policies}\label{sec:unsafe_set}

In the main paper, we showed that $\safeset$ guarantees safety and that we cannot safely expand that set. In this section, we discuss how we can distinguish policies that we know are unsafe from policies we are still uncertain about.
First, we note that, in general, this is impossible. For example, if the expert reward functions are all-zero $\reward_i(\state, \action) = 0$, any policy outside $\safeset$ might still be in $\truesafeset$. However, if we exclude such degenerate cases, we can expect expert policies to lie on the boundary of $\truesafeset$, which gives us some additional information.

Hence, in this section, we assume optimal demonstrations (\Cref{ass:exact-optimality}), and we make an additional (mild) assumption to exclude degenerate demonstrations.

\begin{assumption}[Non-degenerate demonstrations]\label{ass:strictly-optimal}
For each expert reward function $\reward_i$ there are at least two expert policies $\policy_j^*, \policy_l^*$ that have different returns under $\reward_i$, i.e., $\PolicyReturn_{\reward_i}(\policy_j^*) \neq \PolicyReturn_{\reward_i}(\policy_l^*)$.
\end{assumption}

Now, we construct the ``unsafe'' set $\unsafeset = \cup_{i=1}^k \unsafeset_i$ with
\[
\unsafeset_i = \{ x \in \reals^d | \exists \alpha_1, \dots, \alpha_{i-1}, \alpha_{i+1}, \dots \alpha_\NumDemos > 0 \text{ s.t. } x - \FeatureFunction(\policy_i^*) = \sum_{j \neq i} \alpha_j (\FeatureFunction(\policy_i^*) - \FeatureFunction(\policy_j^*)) \} .
\]

\begin{wrapfigure}{R}{0.3\textwidth}\vspace{-1em}
    \includegraphics[width=\linewidth]{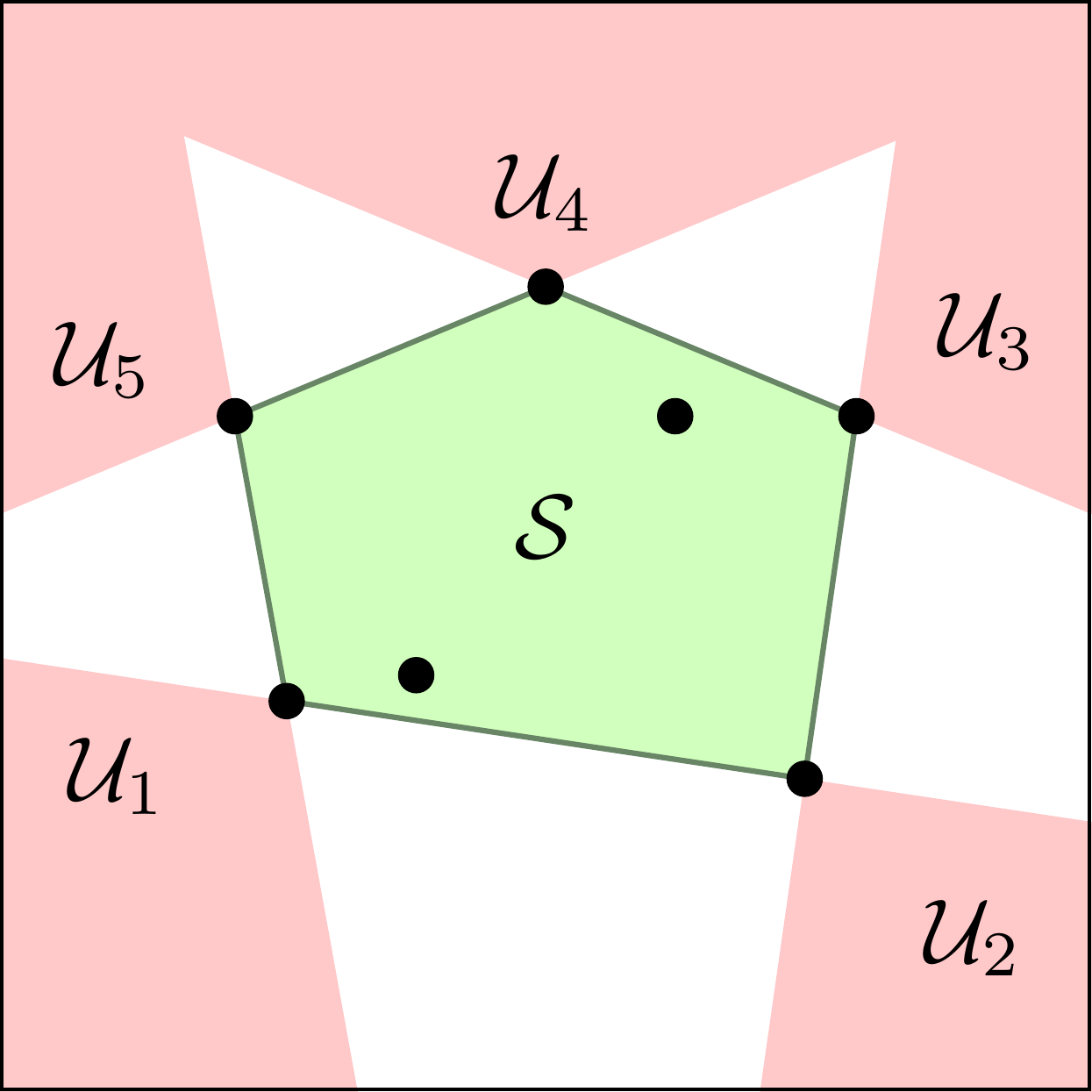}
    \caption{Illustration of safe set and ``unsafe'' set in 2D.}
    \label{fig:safe_unsafe_set}
    \vspace{-2em}
\end{wrapfigure}

For simplicity, we will sometimes write $\policy \in \unsafeset$ instead of $\FeatureFunction(\policy) \in \unsafeset$. At each vertex $i$ of the safe set, $\unsafeset_i$ is a convex cone of points incompatible with observing $\policy_i^*$. In particular, any point in $\unsafeset_i$ is strictly better than $\policy_i^*$, i.e., for any reward function for which $\policy_i^*$ is optimal, a point in $\unsafeset_i$ would be better than $\policy_i^*$. \Cref{fig:safe_unsafe_set} illustrates this construction. We can show that any policy in $\unsafeset$ is guaranteed not to be in $\truesafeset$, and that $\unsafeset$ is the largest possible set with this property.

\begin{restatable}{theorem}{UnsafeSet}\label{thm:unsafe-set}
Under \Cref{ass:exact-optimality} and \Cref{ass:strictly-optimal}, $\unsafeset \cap \truesafeset = \emptyset$. In particular, any $\policy\in\unsafeset$ is not in $\truesafeset$.
\end{restatable}

\begin{proof}
Assume there is a policy $\policy \in \unsafeset\cap\truesafeset$. Then, by construction, there is an $1 \leq i \leq \NumDemos$ such that $\policy \in \unsafeset_i$, i.e., there are $\alpha_1, \dots, \alpha_{i-1}, \alpha_{i+1}, \dots, \alpha_\NumDemos > 0$, such that
\[
\FeatureFunction(\policy) - \FeatureFunction(\policy_i^*) = \sum_{j \neq i} \alpha_j (\FeatureFunction(\policy_i^*) - \FeatureFunction(\policy_j^*)) .
\]

We know that $\policy_i^*$ is optimal for a reward function $\reward_i(\state, \action) = \theta_i^T \FeatureFunction(\state, \action)$ parameterized by $\theta_i \in \reals^d$. So, for any other demonstration $\policy_j^*$, we necessarily have $\theta_i^T (\FeatureFunction(\policy_i^*) - \FeatureFunction(\policy_j^*)) \geq 0$.

Because $\policy \in \unsafeset_i$, we have
\[
\theta_i^T (\FeatureFunction(\policy) - \FeatureFunction(\policy_i^*)) = \sum_{j \neq i} \alpha_j \theta_i^T (\FeatureFunction(\policy_i^*) - \FeatureFunction(\policy_j^*)) \geq 0 .
\]

However, by \Cref{ass:strictly-optimal}, at least one $j$ has $\PolicyReturn_{\reward_i}(\policy_i^*) > \PolicyReturn_{\reward_i}(\policy_j^*)$. And, because we sum over all $j$ with all $\alpha_j > 0$, this implies $\theta_i^T \FeatureFunction(\policy) > \theta_i^T \FeatureFunction(\policy_i^*)$. However, this is a contradiction with $\policy_i^*$ being optimal for reward $\reward_i$ within $\truesafeset$. Hence, $\unsafeset \cap \truesafeset = \emptyset$.
\end{proof}

\begin{restatable}[Set $\unsafeset$ is maximal]{theorem}{UnsafeSetOptimal}\label{thm:unsafe-set-optimal}
Under \Cref{ass:exact-optimality} and \Cref{ass:strictly-optimal}, if $\policy \notin \unsafeset$, there are reward functions $\reward_1, \dots, \reward_\NumDemos$, cost functions $\cost_1, \dots, \cost_\NumConst$, and thresholds $\threshold_1, \dots, \threshold_\NumConst$, such that, each policy $\policy^*_i$ is optimal in the CMDP $(\StateSpace, \ActionSpace, \TransitionModel, \initdist, \DiscountFactor, \reward_i, \{ \cost_j \}_{j=1}^\NumConst, \{ \threshold_j \}_{j=1}^\NumConst)$, and $\policy$ is feasible, i.e., $\policy\in\truesafeset$.
\end{restatable}

\begin{proof}
If $\policy \in \safeset$, the statement follows from \Cref{thm:estimated-cmdp}. So, it remains to show the statement holds when $\policy \notin \safeset$ and $\policy \notin \unsafeset$.
Consider a CMDP with the true safe set
\[
\truesafeset = \{ \policy' | \FeatureFunction(\policy') \in \conv(\FeatureFunction(\policy^*_1), \dots, \FeatureFunction(\policy^*_\NumDemos), \FeatureFunction(\policy)) \} .
\]
As this is a convex polyhedron, we can write it in terms of linear equations (similar to \Cref{thm:estimated-cmdp}); so, there is a CMDP with this true safe set.

We can prove the result, by finding reward functions $\reward_1, \dots, \reward_\NumDemos$ such that the expert policies $\policy_i^* \in \argmax_{\policy\in\truesafeset} \PolicyReturn_{\reward_i}(\policy)$ are optimal, but $\policy \notin \argmax_{\policy\in\truesafeset} \PolicyReturn_{\reward_i}(\policy)$ is not optimal for any of the rewards. With this choice, we will observe the same expert policies $\policy_1^*, \dots, \policy_\NumDemos^*$; but the true safe set has an additional vertex $\policy$, which we never observe.

Now, we show that it is possible for any $\policy \notin \unsafeset$. This means, we cannot expand $\unsafeset$ while guaranteeing that all policies in it are unsafe.
The key is to show that all points $\FeatureFunction(\policy_1^*), \dots, \FeatureFunction(\policy_\NumDemos^*), \FeatureFunction(\policy)$ are vertices of $\truesafeset$, which is the case only if $\policy \notin \unsafeset$.

We can assume w.l.o.g that the $\NumDemos$ demonstrations are vertices of $\conv(\FeatureFunction(\policy_1^*), \dots, \FeatureFunction(\policy_\NumDemos^*))$. So we only need to show that no $\FeatureFunction(\policy_i^*)$ is a convex combination of $\{ \FeatureFunction(\policy_j^*) \}_{j \neq i} \cup \{ \FeatureFunction(\policy) \}$. Assume this was true, i.e., there are $\alpha_j > 0$, $\alpha > 0$ with $\alpha + \sum_{j \neq i} \alpha_j = 1$ and 
\[
\FeatureFunction(\policy_i^*) = \alpha \FeatureFunction(\policy) + \sum_{j \neq i} \alpha_j \FeatureFunction(\policy_j^*) .
\]
Then, we could rewrite
\begin{align*}
\FeatureFunction(\policy) &= \frac{1}{\alpha} \FeatureFunction(\policy_i^*) - \sum_{j\neq i} \frac{\alpha_j}{\alpha} \FeatureFunction(\policy_j^*) \\
&= \frac{1}{\alpha} \FeatureFunction(\policy_i^*) - \frac{1-\alpha}{\alpha} \FeatureFunction(\policy_i^*) + \frac{1-\alpha}{\alpha} \FeatureFunction(\policy_i^*) - \sum_{j\neq i} \frac{\alpha_j}{\alpha} \FeatureFunction(\policy_j^*) \\
&= \FeatureFunction(\policy_i^*) + \sum_{j\neq i} \frac{\alpha_j}{\alpha} (\FeatureFunction(\policy_i^*) - \FeatureFunction(\policy_j^*)) .
\end{align*}
This implies that $\policy \in \unsafeset_i$ because $\alpha_j/\alpha > 0$ for all $j$. However, this is a contradiction to $\policy \notin \unsafeset$; thus, $\FeatureFunction(\policy_1^*), \dots, \FeatureFunction(\policy_\NumDemos^*), \FeatureFunction(\policy)$ are vertices of $\truesafeset$.

Given that the optimal policies are vertices of $\truesafeset$, we can choose $\NumDemos$ linear reward functions as follows:
\begin{enumerate}
    \item For $i$, let $A_i = \conv(\FeatureFunction(\policy_1^*), \dots, \FeatureFunction(\policy_{i-1}^*), \FeatureFunction(\policy_{i+1}^*), \dots, \FeatureFunction(\policy_\NumDemos^*), \FeatureFunction(\policy))$ and $B_i = \{ \FeatureFunction(\policy_i^*) \}$.
    \item $A_i$ and $B_i$ are disjoint convex sets and by the separating hyperplane theorem, we can find a hyperplane that separates them.
    \item Choose $\theta_i$ as the orthogonal vector to the separating hyperplane, oriented from $A_i$ to $B_i$.
\end{enumerate}
With this construction $\policy_i^*$ is optimal only for $\reward_i(s, a) = \theta_i \FeatureFunction(s, a)$ and $\policy$ is optimal for none of the reward functions.
Therefore, any $\policy \notin \unsafeset$ could be an additional vertex of the true safe set, that we did not observe in the demonstrations. In other words, we cannot expand $\unsafeset$.
\end{proof}

Using this construction we can separate the feature space into policies we know are safe ($\policy\in\safeset$), policies we know are unsafe ($\policy\in\unsafeset$), and policies we are uncertain about ($\policy\notin\safeset\cup\unsafeset$). This could be useful to measure the ambiguity of a set of demonstrations. The policies we are uncertain about could be a starting point for implementing an \emph{active learning} version of \AlgNameShort.

\section{\AlgNameShort Implementation Details}\label{app:implementation_details}

This section discusses a few details of our implementation of \AlgNameShort. In particular, we highlight practical modifications related to constructing the convex hull $\safeset$: how we handle degenerate sets of demonstrations and how we greedily select which points to use to construct the convex hull. \Cref{alg:full-algorithm} shows full pseudocode for our implementation of \AlgNameShort.

\subsection{Constructing the Convex Hull}

\begin{algorithm}[t]
\caption{\AlgNameLong}\label{alg:full-algorithm}
\begin{algorithmic}[1]
\Function{\texttt{constraint\_learning}}{$\demonstrations$, $n_{\text{points}}$, $d_{\text{stop}}$}
    \State $(i, d_{\text{next}}) \gets (0, \infty)$
    \State $x_{\text{next}} \gets $ random starting point from $\demonstrations$
    \While{$i \leq n_{\text{points}}$ and $|\demonstrations| > 0$ and $d_{\text{next}} > d_{\text{stop}}$}
        \State $\bar{\demonstrations} \gets \bar{\demonstrations} \cup \{ x_{\text{next}} \}$, $\demonstrations \gets \demonstrations \setminus \{ x_{\text{next}} \}$, $i \gets i + 1$
        \State $\safeset \gets \mathtt{convex\_hull}(\bar{\demonstrations})$
        \State $(x_{\text{next}}, d_{\text{next}}) \gets \mathtt{furthest\_point}(\demonstrations, \safeset)$
    \EndWhile
    \State $\unsafeset \gets \mathtt{unsafe\_set}(\safeset)$
    \State \Return $\safeset$, $\unsafeset$
\EndFunction
\vspace{1em}
\Function{\texttt{furthest\_point}}{$\mathcal{X}, \mathcal{P}$}
    \State $d_{\text{max}} = - \infty$
    \For{$x \in \mathcal{X}$}
        \State Solve QP to find $d \in \text{argmin}_{p \in \mathcal{P}} (x - p)^2$
        \If{$d > d_{\text{max}}$}
            $d_{\text{max}} \gets d$, $x_{\text{max}} \gets x$
        \EndIf
    \EndFor
    \State \Return $(x_{\text{max}}, d_{\text{max}})$
\EndFunction
\vspace{1em}
\Function{convex\_hull}{$\mathcal{X}$}
    \If{If $| \mathcal{X} | < 3$} \Return special case solution $A, \vb$ \EndIf
    \State Determine effective dimension of $\mathcal{X}$
    \State $\mathcal{X}_{\text{proj}} \gets$ project $\mathcal{X}$ to lower dimensional space
    \State $H_{\text{proj}} \gets$ call Qhull to obtain convex hull  of $\mathcal{X}_{\text{proj}}$
    \State $A_{\text{proj}}, \vb_{\text{proj}} \gets$ call Qhull to obtain linear equations describing $H_{\text{proj}}$
    \State $A, \vb \gets$ project $A_{\text{proj}}, \vb_{\text{proj}}$ back to $\reals^d$
    \State \Return $A, \vb$
\EndFunction
\end{algorithmic}
\end{algorithm}

We use the Quickhull algorithm~\citep{barber1996quickhull} to construct convex hulls. In particular, we use a standard implementation, called \emph{Qhull}, which can construct the convex hull from a set of points and determine the linear equations. However, it assumes at least $3$ input points, and non-degenerate inputs, i.e., it assumes that $\rank(\FeatureFunction(\policy_1^*), \dots, \FeatureFunction(\policy_\NumDemos^*)) = d$. To avoid these limiting cases, we make two practical modifications.

\paragraph{Special-case solutions for $1$ and $2$ inputs.} For less than $3$ input points it is easy to construct the linear equations describing the convex hull manually. If we only see a single demonstration $\policy_1^*$, the convex hull is simply, $\conv( \FeatureFunction(\policy_1^*) ) = \{ \FeatureFunction(\policy_1^*) \} = \{ x | A x \leq \vb \}$ with
\[
A = [I_d, -I_d]^T
~~~~~~
\vb = [\FeatureFunction(\policy_1^*)^T, \FeatureFunction(\policy_1^*)^T]^T .
\]
If we observe $2$ demonstrations $\policy_1^*, \policy_2^*$, the convex hull is given by
\[
\conv( \FeatureFunction(\policy_1^*), \FeatureFunction(\policy_2^*) ) = \{ \FeatureFunction(\policy_1^*) + \lambda (\FeatureFunction(\policy_2^*) - \FeatureFunction(\policy_1^*)) | 0 \leq \lambda \leq 1 \} = \{ x | A x \leq \vb \},
\]
where
\[
A = [W, -W, \vv^T, -\vv^T]^T
~~~~~~
\vb = [W \FeatureFunction(\policy_1^*), - W \FeatureFunction(\policy_1^*), \vv^T \FeatureFunction(\policy_2^*), - \vv^T \FeatureFunction(\policy_1^*)]^T .
\]
Here, $\vv = \FeatureFunction(\policy_2^*) - \FeatureFunction(\policy_1^*)$, and $W = [\vw_1, \dots, \vw_{d-1}]$ spans the space orthogonal to $\vv$, i.e., $\vw_i^T \vv = 0$.

\newcommand{\EffectiveDim}{\tilde{d}}

\paragraph{Handling degenerate safe sets.}
If $\rank(\FeatureFunction(\policy_1^*), \dots, \FeatureFunction(\policy_\NumDemos^*)) < d$, we first project the demonstrations to a lower-dimensional subspace in which they are full rank, then we call Qhull to construct the convex hull in this space. Finally, we project back the linear equations describing the convex hull to $\reals^d$. To determine the correct projection, let us define the matrix $D = [\FeatureFunction(\policy_1^*)^T, \dots, \FeatureFunction(\policy_\NumDemos^*)^T]^T \in \reals^{\NumDemos \times d}$. Now, we can do the singular value decomposition (SVD):
$
D = U^T \Sigma V
$,
where $U = [\vu_1, \dots, \vu_\NumDemos] \in \reals^{\NumDemos \times \NumDemos}$, $V = [\vv_1, \dots, \vv_d] \in \reals^{d \times d}$, and $\Sigma = \diag(\sigma_1, \dots, \sigma_d) \in \reals^{\NumDemos \times d}$. Assuming the singular values are ordered in decreasing order by magnitude., we can determine the effective dimension $\EffectiveDim$ of the demonstrations such that $\sigma_{\EffectiveDim} > 0$ and $\sigma_{\EffectiveDim+1} = 0$ (for numerical stability, we use a small positive number instead of $0$). Now, we can define projection matrices
\begin{align*}
X_{\text{ef}} &= [\vu_1, \dots, \vu_{\EffectiveDim}]^T \in \reals^{\EffectiveDim \times d} , \\
X_{\text{orth}} &= [\vu_{\EffectiveDim+1}, \dots, \vu_d]^T \in \reals^{(d-\EffectiveDim) \times d} ,
\end{align*}
where $X_{\text{ef}}$ projects to the span of the demonstrations, and $X_{\text{orth}}$ projects to the complement. Using $X_{\text{ef}}$, we can project the demonstrations to their span and construct the convex hull in that projection, $\conv(X_{\text{ef}} D) = \{ x | A_{\text{proj}} x \leq \vb_{\text{proj}} \}$. To project the convex hull back to $\reals^d$, we construct linear equations in $\reals^d$, such that $\conv(D) = \{ x | A x \leq \vb \}$:
\begin{align*}
A = [ (X_{\text{ef}}^T A_{\text{proj}})^T, X_{\text{orth}}, - X_{\text{orth}} ]^T
~~~~~~
\vb = [\vb_{\text{proj}},  X_{\text{orth}} \FeatureFunction(\policy_1^*), - X_{\text{orth}} \FeatureFunction(\policy_1^*)] ,
\end{align*}
where $X_{\text{ef}}^T A_{\text{proj}}$ projects the linear equations back to $\reals^d$, and the other two components ensure that $X_{\text{orth}} x =  X_{\text{orth}} \FeatureFunction(\policy_1^*)$, i.e., restricts the orthogonal components to the convex hull. In practice, we change this constraint to $- X_{\text{orth}} \FeatureFunction(\policy_1^*) - \epsilon \leq X_{\text{orth}} x \leq X_{\text{orth}} \FeatureFunction(\policy_1^*) + \epsilon$ for some small $\epsilon > 0$.

\subsection{Iteratively Adding Points}

Constructing the safe set from all demonstrations can sometimes be problematic, especially if it results in too many inferred constraints. This scenario is more likely to occur when demonstrations are clustered closely together in feature space, which occurs, e.g., if they are not exactly optimal. To mitigate such problems, we adopt an iterative approach for adding points to the safe set. We start with a random point from the set of demonstrations, and subsequently, we iteratively add the point farthest away from the existing safe set. We can determine the distance of a given demonstration $\policy$ by solving the quadratic program
\[
d \in \argmin_{x \in \safeset} (x - \FeatureFunction(\policy))^2 .
\]
We solve this problem for each remaining demonstration from the safe set to determine which point to add next.
We stop adding points once the distance of the next point we would add gets too small. By changing the stopping distance as a hyperparameter, we can trade-off between adding all points important for expanding the safe set and regularizing the safe set.

\section{Details on the Baselines}\label{app:irl-baselines}

In this section, we discuss the adapted IRL and IL baselines we present in the main paper in more detail.

\subsection{IRL Baselines}

In this section, we discuss how we adapt maximum margin IRL and maximum entropy IRL to implement \emph{Average IRL}, \emph{Shared Reward IRL}, and \emph{Known Reward IRL}.

Recall that these three variants model the expert reward function as one of:
\begin{align*}
&\hat{r}_i(s, a)  &\text{(Average IRL)} \\
&\hat{r}_i(s, a) = r^{\text{known}}_i(s, a) + \hat{\cost}(s, a)  &\text{(Known Reward IRL)} \\
&\hat{r}_i(s, a) = \hat{r}_i(s, a) + \hat{\cost}(s, a)  &\text{(Shared Reward IRL)}
\end{align*}

\subsubsection{Maximum-Margin IRL}\label{app:maximum_margin_irl}

Let us first recall the standard maximum-margin IRL problem~\citep{ng2000algorithms}. We want to infer a reward function $\hat{\reward}$ from an expert policy $\policy_E$, such that the expert policy is optimal under the inferred reward function. To be optimal is equivalent to $\hat{V}^{\policy_E}(\state) \geq \hat{Q}^{\policy_E}(\state, \action)$ for all states $\state$ and actions $\action$. Here $\hat{V}^{\policy_E}$ is the value function of the expert policy computed under reward $\hat{\reward}$ and $\hat{Q}^{\policy_E}(\state, \action)$ is the corresponding Q-function.

However, this condition is underspecified. In maximum margin IRL, we resolve the ambiguity between possible solutions by looking for the inferred reward function $\hat{\reward}$ that maximizes $\sum_{\state,\action} ( \hat{V}^{\policy_E}(\state) - \hat{Q}^{\policy_E}(\state,\action) )$, i.e., the \emph{margin} by which $\policy_E$ is optimal.

The goal of maximum margin IRL is to solve the optimization problem

\begin{equation}
\begin{aligned}
& \underset{\hat{\reward}, \zeta}{\text{maximize}}
& & \sum_i \zeta_i \\
& \text{subject to}
& & \hat{V}^{\policy_E}(\state) - \hat{Q}^{\policy_E}(\state, \action) \geq \zeta, ~~\forall (\state, \action) \in \StateSpace \times \ActionSpace \\
& & & -1 \leq \hat{\reward}(\state, \action) \leq 1
\end{aligned}
\end{equation}

In a tabular environment, we can write this as a linear program because both $\hat{V}^{\policy_E}$ and $\hat{Q}^{\policy_E}$ are linear in the state-action occupancy vectors.

For simplicity, let us consider only state-dependent reward functions $\reward(\state, \action, \state') = \reward(\state')$. Also, let us introduce the vector notation

\begin{align*}
\vr &\defeq (\reward(\state_1), \reward(\state_2), \dots)^T \in \reals^{\StateSpaceSize} \\
\vV^\policy &\defeq (V(\state_1), V(\state_2), \dots)^T \in \reals^{\StateSpaceSize} \\
\vQ^\policy_\action &\defeq (Q(\state_1, \action), Q(\state_2, \action), \dots)^T \in \reals^{\StateSpaceSize} \\\\
P^\policy &\defeq
\begin{pmatrix}
\sum_\action P(\state_1 | \state_1, \action) \policy(\action | \state_1) & \sum_\action P(\state_1 | \state_2, \action) \policy(\action | \state_2) & \dots \\
\sum_\action P(\state_2 | \state_1, \action) \policy(\action | \state_1) & \ddots & \dots \\
\vdots & \dots & \dots
\end{pmatrix} \in \reals^{\StateSpaceSize \times \StateSpaceSize} \\\\
P^\action &\defeq P^{\policy_\action}\text{ 
 with  }\policy_\action(\action | \state) = 1
\end{align*}

Now, we can write Bellman-like identities in this vector notation:
\begin{align*}
\vV^\policy &= P^\policy (\vr + \DiscountFactor \vV^\policy) \\
\vV^\policy &= (I - \DiscountFactor  P^\policy)^{-1} P^\policy \vr \\
\vQ^\policy_\action &= P^\action (\vr + \DiscountFactor \vV^\policy)
\end{align*}

We can also write the maximum margin linear program using the matrix notation:
\begin{equation}
\begin{aligned}
& \underset{\hat{\vr}, \zeta}{\text{maximize}}
& & \sum_i \zeta_i \\
& \text{subject to}
& & \zeta - (P^\policy - P^a) D^\policy \hat{\vr} \leq 0 \\
& & & || \vr ||_\infty \leq 1
\end{aligned}
\tag{IRL-LP}
\label{eq:IRL-LP}
\end{equation}

where $D^\policy = (I + \DiscountFactor (I - \DiscountFactor  P^\policy)^{-1} P^\policy)$.

\paragraph{Average IRL.} The standard maximum margin IRL problem infers a reward function of a single expert policy. However, it cannot learn from more than one policy. To apply it to our setting, we infer a rubeward function $\hat{\reward}_i$ for each of the demonstrations $\policy_i^*$. Then for any new reward function $\evalreward$, we add the average of our inferred rewards, i.e., we optimize for $\evalreward(\state) + \frac{1}{\NumDemos} \sum_i \hat{\reward}_i(\state)$. So, implicitly, we assume the reward component of the inferred rewards is zero in expectation, and we can use the average to extract the constraint component of the inferred rewards.

\paragraph{Shared Reward IRL.} We can also extend the maximum margin IRL problem to learn from multiple expert policies.
For multiple policies, we can extend the maximum margin IRL problem to infer multiple reward function simultaneously:
\begin{equation*}
\begin{aligned}
& \underset{\hat{\reward}_1, \dots \hat{\reward}_n, \zeta}{\text{maximize}}
& & \sum_i \zeta_i \\
& \text{subject to}
& & \hat{V}_1^{\policy_1}(\state) - \hat{Q}_1^{\policy_1}(\state, \action) \geq \zeta_1, ~~\forall (\state, \action) \in \StateSpace \times \ActionSpace \\
& & & \hat{V}_2^{\policy_2}(\state) - \hat{Q}_2^{\policy_2}(\state, \action) \geq \zeta_2, ~~\forall (\state, \action) \in \StateSpace \times \ActionSpace \\
& & & \hspace{8em} \vdots \\
& & & \hat{V}_\NumDemos^{\policy_\NumDemos}(\state) - \hat{Q}_\NumDemos^{\policy_\NumDemos}(\state, \action) \geq \zeta_\NumDemos, ~~\forall (\state, \action) \in \StateSpace \times \ActionSpace \\
\end{aligned}
\end{equation*}
where we infer $\NumDemos$ separate reward functions $\hat{\reward}_i$ and $\hat{V}_i$ and $\hat{Q}_i$ are computed with reward function $\hat{\reward}_i$.
For \emph{Shared Reward IRL}, we model the reward functions as $\hat{\reward}_i(\state) = \hat{\reward}^+_i(\state) + \hat{\cost}(\state)$, where $\hat{\reward}^+_i$ is different for each expert policy, and $\hat{\cost}$ is a shared constraint penalty. This gives us the following linear program to solve:
\begin{equation}
\begin{aligned}
& \underset{\hat{\vr}_1, \dots, \hat{\vr}_\NumDemos, \hat{\vc}, \zeta}{\text{maximize}}
& & \sum_i \zeta_i \\
& \text{subject to}
& & \zeta_i - (P^{\policy_i} - P^a) D^{\policy_i} \hat{\vr}_i - (P^{\policy_i} - \vP^a) D^{\policy_i} \hat{\vc} \leq 0 \text{ for each policy } i \\
& & & || \hat{\vr}_i ||_\infty \leq 1 \text{ for each } i \\
& & & || \hat{\vc} ||_\infty \leq 1
\end{aligned}
\tag{IRL-LP-SR}
\label{eq:IRL-LP-SR}
\end{equation}

\paragraph{Known Reward IRL.} Alternatively, we can assume we know the reward component, and only infer the constraint. In that case $\vr_i$ is the known reawrd vector and no longer a decision variable, and we obtain the following LP:
\begin{equation}
\begin{aligned}
& \underset{\hat{\vc}, \zeta}{\text{maximize}}
& & \sum_i \zeta_i \\
& \text{subject to}
& & \zeta_i - (P^{\policy_i} - P^a) D^{\policy_i} \vc \leq (P^{\policy_i} - P^a) D^{\policy_i}\vr_i \text{ for each policy } i \\
& & & || \hat{\vc} ||_\infty \leq 1
\end{aligned}
\tag{IRL-LP-KR}
\label{eq:IRL-LP-KR}
\end{equation}

\subsubsection{Maximum-Entropy IRL}

We adapt Maximum Entropy IRL~\citep{ziebart2008maximum} to our setting by parameterizing the reward function as $\hat{\reward}(\state,\action) = \hat{\theta}_i^T \FeatureFunction(\state, \action) + \hat{\phi}^T \FeatureFunction(\state, \action)$. Here $\hat{\theta}_i$ parameterizes the ``reward'' component for each expert policy, and $\hat{\phi}$ parameterizes the reward function's shared ``constraint'' component.

\Cref{alg:modified-max-ent-irl} shows how we modify the standard maximum entropy IRL algorithm with linear parameterization to learn $\hat{\theta}_i$ and $\hat{\phi}$ simultaneously. Instead of doing gradient updates on a single parameter, we do gradient updates for both $\hat{\theta}_i$ and $\hat{\phi}$ with two different learning rates $\alpha_\theta$ and $\alpha_\phi$. Also, we iterate over all demonstrations $\policy_1^*, \dots, \policy_\NumDemos^*$. This effectively results in updating $\hat{\phi}$ with the average gradient from all demonstrations, while $\hat{\theta}_i$ is updated only from the gradients for demonstration $\policy_i^*$.

Depending on initialization and learning rate, \Cref{alg:modified-max-ent-irl} implements all of our IRL-based baselines:
\begin{itemize}
    \item \makebox[3.2cm]{Average IRL:\hfill} $\alpha_\theta > 0$, $\alpha_\phi = 0$, $\hat{\theta}_0 = 0$, $\hat{\phi}_0 = 0$
    \item \makebox[3.2cm]{Shared Reward IRL:\hfill} $\alpha_\theta > 0$, $\alpha_\phi > 0$, $\hat{\theta}_0 = 0$, $\hat{\phi}_0 = 0$
    \item \makebox[3.2cm]{Known Reward IRL:\hfill} $\alpha_\theta = 0$, $\alpha_\phi > 0$, $\hat{\theta}_i = \theta_i^{\text{known}}$, $\hat{\phi}_0 = 0$
\end{itemize}

For Known Reward IRL and Shared Reward IRL, we can apply the learned $\hat{\phi}$ to a new reward function by computing $\evalreward(\state, \action) + \hat{\phi}^T \FeatureFunction(\state, \action)$. For Average IRL, we use $\evalreward(\state, \action) + \frac{1}{\NumDemos} \sum_i \hat{\theta}_i^T \FeatureFunction(\state, \action)$.

\begin{algorithm}[t]
\caption{Maximum Entropy IRL, adapted to constraint learning.}
\begin{algorithmic}[1]
\Require Expert policies $\policy_1^*, \dots, \policy_\NumDemos^*$, learning rates $\alpha_\theta$, $\alpha_\phi$, initial parameters $\hat{\theta}_0, \hat{\phi}_0 \in \mathbb{R}^d$
\State Initialize:
    \hspace{0.2em}
    $\hat{\theta}_i \gets \hat{\theta}_0$ for $i = 1, \dots, \NumDemos$,
    \hspace{0.2em}
    $\hat{\phi} \gets \hat{\phi}_0$,
    \hspace{0.2em}
    $i \gets 0$
\While{not converged}
  \State $i \gets \mathrm{mod}(i + 1, \NumDemos + 1)$
  \State Update policy $\hat{\policy}$ for
  \hspace{0.2em}
  $
  \hat{\reward}(\state, \action) = \hat{\theta}_i^T \FeatureFunction(\state, \action) + \hat{\phi}^T \FeatureFunction(\state, \action)
  $
  \State Compute the gradient:
    \hspace{0.2em}
    $\nabla = \FeatureFunction(\policy_i^*) - \FeatureFunction(\hat{\policy})$
  \State Update weights:
    \hspace{0.2em}
    $\hat{\theta}_i \gets \hat{\theta}_i + \alpha_\theta \nabla$, 
    \hspace{0.2em}
    $\hat{\phi} \gets \hat{\phi} + \alpha_\phi \nabla$
\EndWhile
\State \Return $\hat{\theta}_1, \dots, \hat{\theta}_\NumDemos$, $\hat{\phi}$
\end{algorithmic}
\label{alg:modified-max-ent-irl}
\end{algorithm}

\subsection{Imitation Learning Baseline}

\Cref{alg:imitation-learning-baseline} shows the general imitation learning baseline discussed in the main paper. We can implement it using any imitation learning algorithm. The approach simply runs imitation learning on each demonstration and returns the best policy for a new evaluation reward.

In our experiments, we use synthetic demonstrations. Hence, for each demonstration, we know exactly which policy produced it, which allows us to implement the idealized IL baseline presented in the main paper.

This idealized version of \Cref{alg:imitation-learning-baseline} is guaranteed to return safe policies because all demonstrations are safe. Also, if we assume linear reward functions and have no task- or environment-transfer, this algorithm returns the same solutions as \AlgNameShort. For this restricted setting, we could also obtain convergence guarantees for this algorithm.

However, for task and environment transfer this is no longer true. Our experiments show that \AlgNameShort can, e.g., learn to drive straight from a set of demonstrations that always turn left or right. The IL baseline clearly fails this task, because it only imitates turning left or right. Further, our experiments show that \AlgNameShort can learn to generalize from defensive drivers to aggressive drivers. Again, the IL baseline fails because it directly learns policies that do not generalize.

\begin{algorithm}[t]
\caption{Imitation learning baseline.}
\begin{algorithmic}[1]
\Require Expert policies $\policy_1^*, \dots, \policy_\NumDemos^*$, evaluation reward $\evalreward$
\State $\hat{\demonstrations} \gets \{\}$
\For{$i = 1, \dots, \NumDemos$}
    \State $\hat{\policy}_i^* \gets \mathtt{imitation\_learning}(\policy_i^*)$
    \State $\hat{\demonstrations} \gets \hat{\demonstrations} \cup \{ \hat{\policy}_i^* \}$
\EndFor
\Return $\argmax_{\hat{\policy}_i^* \in \hat{\demonstrations}} \PolicyReturn_{\text{eval}}(\hat{\policy}_i^*)$
\end{algorithmic}
\label{alg:imitation-learning-baseline}
\end{algorithm}

\section{Experiment Details}\label{app:experiment_details}

In this section, we provide more details about our experimental setup in the Gridworld environments (\Cref{app:gridworld-details}) and the driving environment (\Cref{app:highway-details}). In particular, we provide details on the environments, how we solve them, and the computational resources used.

\subsection{Gridworld Experiments}\label{app:gridworld-details}

\paragraph{Environment Details.}
An $N\times N$ Gridworld has $N^2$ discrete states and a discrete action space with $5$ actions: \emph{left}, \emph{right}, \emph{up}, \emph{down}, and \emph{stay}.  Each action corresponds to a movement of the agent in the grid. Given an action, the agent moves in the intended direction except for two cases: (1) if the agent would leave the grid, it instead stays in its current cell, and (2) with probability $p$, the agent takes a random action instead of the intended one.

In the Gridworld environments, we use the state-action occupancies as features and define rewards and costs independently per state. We uniformly sample $n_{\text{goal}}$ and $n_{\text{limited}}$ goal tiles and limited tiles, respectively. Each goal cell has an average reward of $1$; other tiles have an average reward of $0$. Constraints are associated with limited tiles, which the agent must avoid. We uniformly sample the threshold for each constraint. We ensure feasibility using rejection sampling, i.e., we resample the thresholds if there is no feasible policy.
In our experiments, we use $N=10$, $n_{\text{goal}} = 20$, $n_{\text{limited}} = 10$, and we have $\NumConst = 4$ different constraints. We choose discount factor $\DiscountFactor = 0.9$.

While the constraints are fixed and shared for one Gridworld instance, we sample different reward functions from a Gaussian with standard deviation $0.1$ and mean $1$ for goal tiles and $0$ for non-goal tiles. To test reward transfer, we sample a different set of goal tiles from the grid for evaluation.

For the experiment without constraint transfer, we choose $p=0.2$, i.e., a deterministic environment. To test transfer to a new environment, we use $p=0.2$ during training and $p=0$ during evaluation.

\paragraph{Linear Programming Solver.} We can solve tabular environments using linear programming~\citep{altman1999constrained}. The idea is to optimize over the occupancy measure $\Occupancy_\policy(\state, \action)$. Then, the policy return is a linear objective, and CMDP constraints become additional linear constraints of the LP.

The other constraint we need is a Bellman equation on the state-action occupancy measure:
\begin{align*}
\Occupancy_\policy(\state, \action)
&= \Expectation_{\TransitionModel,\policy} \left[ \sum_{t=0}^\infty \DiscountFactor^t \indicator{\state_t=\state, \action_t=\action} \right]
= \sum_{t=0}^\infty \DiscountFactor^t \TransitionModel(\state_t = \state, \action_t = \action | \policy) \\
&= \initdist(\state, \action) + \DiscountFactor \sum_{t=0}^\infty \DiscountFactor^t \TransitionModel(\state_{t+1} = \state, \action_{t+1} = \action | \policy) \\
&= \initdist(\state, \action) + \DiscountFactor \sum_{t=0}^\infty \DiscountFactor^t \sum_{\state', \action'} \TransitionModel(\state | \state', \action') \policy(\action | \state) \TransitionModel(\state_t = \state', \action_t = \action' | \policy) \\
&= \initdist(\state, \action) + \DiscountFactor \sum_{\state', \action'} \TransitionModel(\state | \state', \action') \policy(\action | \state) \Occupancy_\policy(\state', \action')
\end{align*}

Using this, we can formulate solving a CMDP as solving the following LP:
\begin{equation}
\begin{aligned}
    & \underset{\Occupancy_\policy(\state, \action)}{\text{maximize}}
    & & \sum_{\state, \action} \Occupancy_\policy(\state, \action) \reward(\state, \action) \\
    & \text{subject to}
    & & \sum_\action \Occupancy_\policy(\state, \action) = \initdist(\state) + \DiscountFactor \sum_{\state', \action'} \TransitionModel(\state | \state', \action') \Occupancy_\policy(\state', \action'), \\
    & & & \sum_{\state, \action} \Occupancy_\policy(\state, \action) \cost_j(\state, \action) \leq \threshold_j, \text{ for each constraint } j \\
    & & & \Occupancy_\policy(\state, \action) \geq 0,
\end{aligned}
\tag{CMDP-LP}
\label{eq:CMDP-LP}
\end{equation}
By solving this LP, we obtain the state-occupancy measure of an optimal policy $\Occupancy_{\policy^*}$ which we can use to compute the policy as $\policy^*(\action | \state) = \Occupancy_{\policy^*}(\state, \action) / ( \sum_{\action'} \Occupancy_{\policy^*}(\state, \action') )$.

\paragraph{Computational Resources.}
We run each experiment on a single core of an Intel(R) Xeon(R) CPU E5-2699 v3 processor. The experiments generally finish in less than $1$ minute, often finishing in a few seconds. For each of the $3$ transfer settings, we run $100$ random seeds for $26$ different numbers of demonstrations $\NumDemos$, i.e., for each method, we run $7800$ experiments.

\subsection{Driving Experiments}\label{app:highway-details}

\begin{wrapfigure}{R}{0.3\textwidth}\centering
    \includegraphics[width=\linewidth]{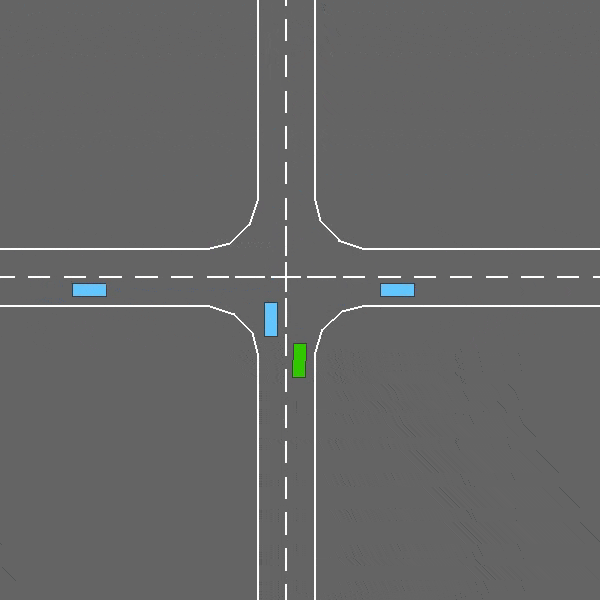}
    \caption{\texttt{highway-env} intersection environment.}
    \label{fig:intersection}
    \vspace{-3em}
\end{wrapfigure}

\paragraph{Environment Details.}
We use the intersection environment provided by \texttt{highway-env}~\citep{highway-env} shown in \Cref{fig:intersection}. The agent controls the green car, while the environment controls the blue cars.
The action space has high-level actions for speeding up and slowing down and three actions for choosing one of three trajectories: turning left, turning right, and going straight. The observation space for the policy contains the position and velocity of the agent and other cars.

The default environment has a well-tuned reward function with positive terms for reaching a goal and high velocity and negative terms to avoid crashes and staying on the road. We change the environment to have a reward function that rewards reaching the goal and includes driver preferences on velocity and heading angle and multiple cost functions that limit bad events, such as driving too fast or off the road.
To define the reward and cost functions, we define a set of features:
\[
\FeatureFunction(\state, \action) =
\begin{pmatrix}
&{ \color{ForestGreen} f_{\text{left}} }(\state, \action) \\
&{ \color{ForestGreen} f_{\text{straight}} }(\state, \action) \\
&{ \color{ForestGreen} f_{\text{right}} }(\state, \action) \\
&{ \color{ForestGreen} f_{\text{vel}} }(\state, \action) \\
&{ \color{ForestGreen} f_{\text{heading}} }(\state, \action) \\
&{ \color{BrickRed} f_{\text{toofast}} }(\state, \action) \\
&{ \color{BrickRed} f_{\text{tooclose}} }(\state, \action) \\
&{ \color{BrickRed} f_{\text{collision}} }(\state, \action) \\
&{ \color{BrickRed} f_{\text{offroad}} }(\state, \action)
\end{pmatrix}
\in \reals^9
\]
where
$
{ \color{ForestGreen} f_{\text{left}} },
{ \color{ForestGreen} f_{\text{straight}} },
{ \color{ForestGreen} f_{\text{right}} }
$
are $1$ if the agent reached the goal after turning left, straight, or right respectively, and $0$ else. ${ \color{ForestGreen} f_{\text{vel}} }$ is the agents velocity and ${ \color{ForestGreen} f_{\text{heading}} }$ is its heading angle.

The last four features are indicators of undesired events. ${ \color{BrickRed} f_{\text{toofast}} }$ is $1$ if the agent exceeds the speed limit, ${ \color{BrickRed} f_{\text{tooclose}} }$ is $1$ if the agent is too close to another car, ${ \color{BrickRed} f_{\text{collision}} }$ is $1$ if the agent has collided with another car, and ${ \color{BrickRed} f_{\text{offroad}} }$ is $1$ if the agent is not on the road.

The ground truth constraint in all of our experiments is defined by:
\begin{align*}
\phi_1 &= (0, 0, 0, 0, 0, 1, 0, 0, 0)^T \\
\phi_2 &= (0, 0, 0, 0, 0, 0, 1, 0, 0)^T \\
\phi_3 &= (0, 0, 0, 0, 0, 0, 0, 1, 0)^T \\
\phi_4 &= (0, 0, 0, 0, 0, 0, 0, 0, 1)^T
\end{align*}
and thresholds
\[
\threshold_1 = 0.2, ~~
\threshold_2 = 0.2, ~~
\threshold_3 = 0.05, ~~
\threshold_4 = 0.1,
\]
i.e., we have one constraint for each of the last four features, and each constraint restricts how often this feature can occur. For example, collisions should occur in fewer than $5\%$ of steps and speed limit violations in fewer than $20\%$ of steps.

The driver preferences are a function of the first five features. Each driver wants to reach one of the three goals (sampled uniformly) and has a preference about velocity and heading angle sampled from
$\theta_{\text{vel}} \sim \Gaussian(0.1, 0.1)$, and $\theta_{\text{heading}} \sim \Gaussian(-0.2, 0.1)$. These preferences simulate variety between the demonstrations from different drivers.

When inferring constraints, for simplicity, we restrict the feature space to
$(
{ \color{BrickRed} f_{\text{toofast}} },
{ \color{BrickRed} f_{\text{tooclose}} },
{ \color{BrickRed} f_{\text{collision}} },
{ \color{BrickRed} f_{\text{offroad}} }
)$.
Our approach also works for all features but needs more (and more diverse) samples to learn that the other features are safe.

To test constraint transfer to a new reward function, we sample goals in demonstrations to be either ${ \color{ForestGreen} f_{\text{left}} }$ or ${ \color{ForestGreen} f_{\text{right}} }$, while during evaluation the goal is always ${ \color{ForestGreen} f_{\text{straight}} }$.

\looseness -1
To test constraint transfer to a modified environment, we collect demonstrations with more defensive drivers (that keep larger distances to other vehicles) and evaluate the learned constraints with more aggressive drivers (that keep smaller distances to other vehicles).

\paragraph{Driving Controllers.}
We use a family of parameterized driving controllers for two purposes: (1) to optimize over driving policies, and (2) to control other vehicles in the environment.
The controllers greedily decide which trajectory to choose, i.e., choose the goal with the highest reward, and they control acceleration via a linear equation with parameter vector $\omega \in \reals^5$. As the vehicles generally follow a fixed trajectory, these controllers behave similarly to an \emph{intelligent driver model} (IDM).
Specifically, we use linearized controllers proposed by in \citet{leurent2019approximate}; see their Appendix B for details about the parameterization.

\paragraph{Constrained Cross-Entropy Method Solver.}
We solve the driving environment by optimizing over the parametric controllers, using a constrained cross-entropy method~\citep[CEM;][]{rubinstein2004cross,wen2018constrained}. The constrained CEM ranks candidate solutions first by constraint violations and then ranks feasible solutions by reward. It aims first to find a feasible solution and then improve upon it within the feasible set. We extend the algorithm by \citet{wen2018constrained} to handle multiple constraints by ranking first by the number of violated constraints, then the magnitude of the violation, and only then by reward. \Cref{alg:cross_entropy_driver} shows the pseudocode for this algorithm.

\paragraph{Computational Resources.}
The computational cost of our experiments is dominated by optimizing policies using the CEM. The IRL baselines are significantly more expensive because they need to do so in the inner loop. To combat this, we parallelize evaluating candidate policies in the CEM, performing $50$ roll-outs in parallel. We run experiments on AMD EPYC 64-Core processors. Any single experiment using \AlgNameShort finishes in approximately $5$ hours, while any experiment using the IRL baselines finishes in approximately $20$ hours. For each of the $3$ constraint transfer setups, we run $5$ random seeds for $30$ different numbers of demonstrations $\NumDemos$, i.e., we run $450$ experiments in total for each method we test.

\begin{algorithm}[t]
\caption{Cross-entropy method for (constrained) RL, based on \citet{wen2018constrained}.}
\label{alg:cross_entropy_driver}
\begin{algorithmic}[1]
  \Require $n_\text{iter}$, $n_\text{samp}$, $n_\text{elite}$
  \State Initialize policy parameters $\mu\in\reals^d$, $\sigma\in\reals^d.$
  \For{iteration = $1, 2, \dots, n_\text{iter}$}
      \State Sample $n_\text{samp}$ samples of $\omega_i \sim \Gaussian(\mu, \diag(\sigma))$
      \State Evaluate policies $\omega_1, \dots, \omega_{n_\text{samp}}$ in the environment
      \If{constrained problem}
          \State Compute number of constraints violated
          $ N_{\text{viol}}(\omega_i) = \sum_j \indicator{\CumulativeCost_j(\omega_i) > 0} $
          \State Compute total constraint violation
          $T_{\text{viol}}(\omega_i) = \sum_j \max(\CumulativeCost_j(\omega_i), 0)) $
          \State Sort $\omega_i$ in descending order first by $N_{\text{viol}}$, then by $T_{\text{viol}}$
          \State Let $E$ be the first $n_\text{elite}$ policies
          \If{$N_{\text{viol}}(\omega_{n_\text{elite}}) = 0$}
              \State Sort $\{\omega_i | N_{\text{viol}}(\omega_i) = 0\}$ in descending order of return $\PolicyReturn(\omega_i)$
              \State Let $E$ be the first $n_\text{elite}$ policies
          \EndIf
      \Else
          \State Sort $\omega_i$ in descending order of return $G(\omega_i)$
          \State Let $E$ be the first $n_\text{elite}$ policies
      \EndIf
      \State Fit Gaussian distribution with mean $\mu$ and diagonal covariance $\sigma$ to $E$
  \EndFor
  \State \textbf{return} $\mu$
\end{algorithmic}
\end{algorithm}

\section{Additional Results}\label{app:additional_results}

\begin{figure}
\centering
\includegraphics[width=.4\linewidth]{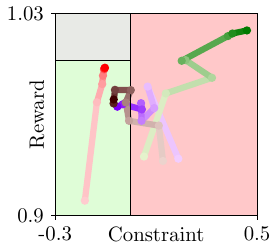}
\caption{Illustration of safety and reward of {\color{\CoCoRLcolor}\AlgNameShort}, {\color{\MaxMarginIRLVanillacolor}Average IRL}, {\color{\MaxMarginIRLSharedRewardcolor}Shared Reward IRL}, and {\color{\MaxMarginIRLKnownRewardcolor}Known Reward IRL} in the Gridworld without constraint transfer. We show the $95$th percentile of the constraint and the corresponding reward. The lines are colored from light to dark to indicate increasing number of demonstrations. The green region is safe, the red region is unsafe, and the grey region would be safe but is unattainable. {\color{\CoCoRLcolor}\AlgNameShort} quickly reaches a safe and near-optimal solution. {\color{\MaxMarginIRLKnownRewardcolor}Known Reward IRL} achieves a high reward but is unsafe. {\color{\MaxMarginIRLSharedRewardcolor}Shared Reward IRL} and {\color{\MaxMarginIRLVanillacolor}Average IRL} are first unsafe and then safe but suboptimal.}
\label{fig:gridworld_reward_constraint_plot}
\end{figure}

In this section, we present additional experimental results that provide more insights on how \AlgNameShort scales with feature dimensionality and how it compares to all variants of the IRL baselines.

\subsection{Validating \AlgNameShort in Single-state CMDPs}\label{app:synthetic-experiments}

As an additional simple test-bed, we consider a single state CMDP, where $\StateSpace = \{\state\}$ and $\ActionSpace = \reals^d$. We remove the complexity of the transition dynamics by having $\TransitionModel(\state | \state, \action) = 1$ for all actions. Also, we choose $\DiscountFactor = 0$, and the feature vectors are simply $\FeatureFunction(\state, \action) = \action$. Given a reward function parameterized by $\theta \in \reals^d$ and constraints parameterized by $\phi_1, \dots, \phi_\NumConst \in \reals^{d}$ and thresholds $\threshold_1, \dots, \threshold_\NumConst$, we can find the best action by solving the linear program $\action^* \in \argmax_{\phi_1^T \action \leq \threshold_1, \dots, \phi_\NumConst^T \action \leq \threshold_\NumConst} \theta^T \action$.

To generate demonstrations, we sample both the rewards $\theta_1, \dots, \theta_\NumDemos$ and the constraints $\phi_1, \dots, \NumConst$ from the unit sphere. We set all thresholds $\threshold_j = 1$. The constraints are shared between all demonstrations.

Our results confirm that \AlgNameShort guarantees safety. We do not see any unsafe solution when optimizing over the inferred safe set. \Cref{fig:synthetic_results} shows the reward we achieve as a function of the number of demonstrations for different dimensions $d$ and different numbers of true constraints $\NumConst$. \AlgNameShort achieves good performance with a small number of samples and eventually approaches optimality. The number of samples depends on the dimensionality, as \Cref{thm:convergence-noise-free} suggests. We find the sample complexity depends less on the number of constraints in the true environment which is consistent with \Cref{thm:convergence-noise-free}.

\subsection{Maximum Margin IRL in Gridworld Environments}\label{app:additional-gridworld-experiments}

In addition to IRL baselines based on maximum entropy IRL, we evaluated methods based on maximum margin IRL (\Cref{app:maximum_margin_irl}). \Cref{fig:additional_gridworld_results} shows results from a $3 \times 3$ gridworld. The environment uses $N=3$, $n_{\text{goal}} = 2$, $n_{\text{limited}} = 3$, and we have $\NumConst = 2$ different constraints.

Like the maximum entropy IRL methods, the maximum margin IRL methods fail to return safe solutions. \Cref{fig:gridworld_reward_constraint_plot} illustrates how the different methods trade-off safety and maximizing rewards.

\begin{figure}
\centering
\begin{subfigure}[b]{0.28\linewidth}
    \centering
    \includegraphics[width=\linewidth]{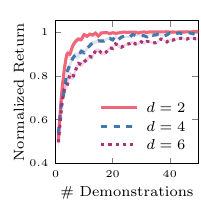}
    \caption{$\NumConst = 8$ constraints}
\end{subfigure}
\hfill
\begin{subfigure}[b]{0.28\linewidth}
    \centering
    \includegraphics[width=\linewidth]{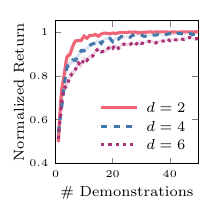}
    \caption{$\NumConst = 12$ constraints}
\end{subfigure}
\hfill
\begin{subfigure}[b]{0.28\linewidth}
    \centering
    \includegraphics[width=\linewidth]{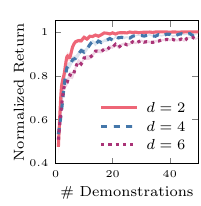}
    \caption{$\NumConst = 16$ constraints}
\end{subfigure}
\caption{
Return achieved by \AlgNameShort in single-state CMDPs for different dimensions $d$ and numbers of constraints $\NumConst$. We plot the mean and standard errors over $100$ random seeds, although the standard errors are nearly $0$ everywhere. As expected, we need more demonstrations to approximate the true safe set well in higher dimensional feature spaces. There is no noticeable difference in sample complexity for different numbers of constraints. All solutions returned by \AlgNameShort are safe.
}
\label{fig:synthetic_results}
\end{figure}

\begin{figure}
\centering
\includegraphics[width=0.95\linewidth]{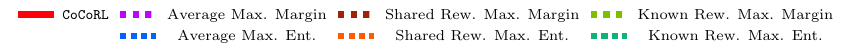} \\
\hspace{-1.5em}
\begin{subfigure}[b]{0.329\linewidth}
    \centering
    \includegraphics[width=1.03\linewidth]{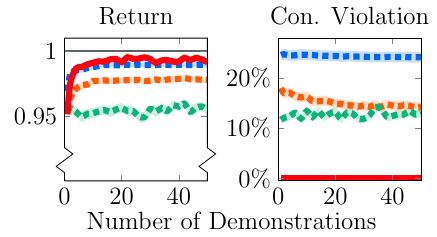} \\
    \includegraphics[width=1.03\linewidth]{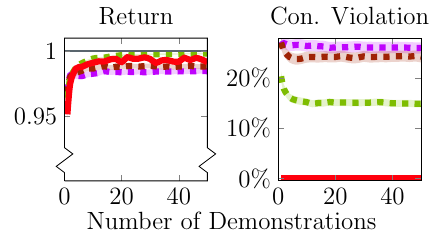}
    \caption{Single environment.}
    \label{subfig:gridworld_exp1_full}
\end{subfigure}
\begin{subfigure}[b]{0.329\linewidth}
    \centering
    \includegraphics[width=1.03\linewidth]{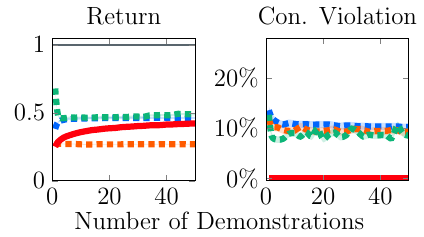} \\
    \includegraphics[width=1.03\linewidth]{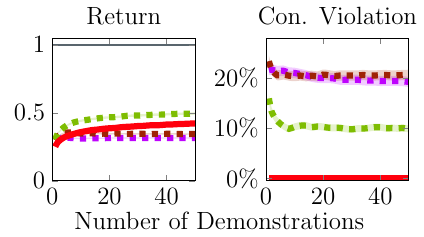}
    \caption{Transfer to new task.}
    \label{subfig:gridworld_exp2_full}
\end{subfigure}
\begin{subfigure}[b]{0.329\linewidth}
    \centering
    \includegraphics[width=1.03\linewidth]{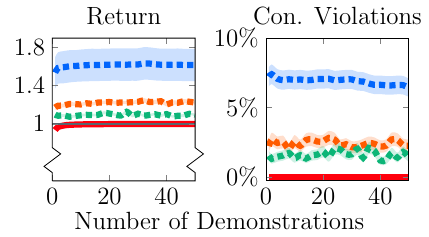} \\
    \includegraphics[width=1.03\linewidth]{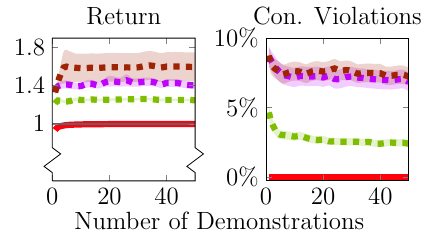}
    \caption{Transfer to new environment.}
    \label{subfig:gridworld_exp3_full}
\end{subfigure}
\caption{
Experimental results in $3 \times 3$ Gridworld environments, with maximum entropy and maximum margin IRL baselines. The top row of plots shows \AlgNameShort compared to the maximum entropy IRL baselines. The bottom row of plots shows \AlgNameShort compared to the maximum-margin IRL baselines (\Cref{app:maximum_margin_irl}). The plots show mean and standard errors over 100 random seeds. The results are qualitatively similar: the IRL baselines are not safe and can return poor solutions. The maximum margin IRL baselines are slightly better than the maximum entropy IRL baselines. in terms of return but worse in terms of safety.
}
\label{fig:additional_gridworld_results}
\vspace{2em}
\end{figure}

\end{document}